\documentclass{article}

\usepackage{fullpage}

\usepackage[utf8]{inputenc} % allow utf-8 input
\usepackage[T1]{fontenc}    % use 8-bit T1 fonts
\usepackage[colorlinks = true, citecolor = magenta]{hyperref}
\usepackage{url}            % simple URL typesetting
\usepackage{booktabs}       % professional-quality tables
\usepackage{amsfonts}       % blackboard math symbols
\usepackage{nicefrac}       % compact symbols for 1/2, etc.
\usepackage{microtype}      % microtypography
\usepackage{xcolor}         % colors

\usepackage{graphicx}
\usepackage{subfigure}
\usepackage{booktabs} % for professional tables
\usepackage{amsmath}
\usepackage{dsfont}
\usepackage{amsthm}
\usepackage{mathtools}
\usepackage{tikz}
\usepackage{complexity}
\usepackage{bbm}
\usepackage{algorithm}
\usepackage{algorithmic}
\usepackage[bottom]{footmisc}
\theoremstyle{definition}
\newtheorem{definition}{Definition}
\newtheorem{lemma}{Lemma}

\DeclareMathOperator*{\argmax}{arg\,max}
\DeclareMathOperator*{\argmin}{arg\,min}

\newtheorem{remark}{Remark}
\newtheorem{corollary}{Corollary}
\newtheorem{proposition}{Proposition}
\newtheorem{theorem}{Theorem}
\newcommand{\myquad}[1][1]{\hspace*{#1em}\ignorespaces}
\newcommand*\circled[1]{\tikz[baseline=(char.base)]{
            \node[shape=circle,draw,inner sep=2pt] (char) {#1};}}

\usepackage{nccmath}
\usepackage[round]{natbib}

\title{Polynomial Time Reinforcement Learning in Factored State MDPs with Linear Value Functions\footnotetext[0]{This work appeared at AISTATS 2022. An early version also appeared at the Neurips 2021 Workshop on Ecological Theory of Reinforcement Learning.}}

\author{%
   Zihao Deng\thanks{Equal contribution, alphabetical ordering.} \\
   Department of Computer Science \& Engineering \\
   Washington University in St.\ Louis \\
   \texttt{zihao.deng@wustl.edu} \\
   \and
   Siddartha Devic\footnotemark[1] \thanks{Work performed while affiliated with UT Dallas and participating in an NSF REU at Washington University in St.\ Louis.} \\
   Department of Computer Science\\
   University of Southern California\\
   \texttt{devic@usc.edu} \\
   \and
   Brendan Juba\footnotemark[1]\\
   Department of Computer Science \& Engineering \\
   Washington University in St.\ Louis \\
   \texttt{bjuba@wustl.edu}
}

\begin{document}
\maketitle

\begin{abstract}
Many reinforcement learning (RL) environments in practice feature enormous state spaces that may be described compactly by a ``factored'' structure, that may be modeled by Factored Markov Decision Processes (FMDPs).
We present the first polynomial-time algorithm for RL in Factored State MDPs (generalizing FMDPs) that neither relies on an oracle planner nor requires a linear transition model; it only requires a linear value function with a suitable local basis with respect to the factorization, permitting efficient variable elimination.
With this assumption, we can solve this family of Factored State MDPs in polynomial time by constructing an efficient separation oracle for convex optimization.
Importantly, and in contrast to prior work on FMDPs, we do not assume that the transitions on various factors are conditionally independent.
\end{abstract}

% \blfootnote{\printfnsymbol{1}Indicates alphabetical ordering.}

\section{Introduction}
%Reinforcement learning (RL) bounds in polynomial sized Markov Decision Processes (MDPs) are well studied. However, 
Many important application domains of Reinforcement learning (RL) -- such as resource allocation or complex games -- 
%may in fact have an exponential sized state space wherein many theoretical guarantees fall apart. 
feature large state spaces, for which existing theoretical guarantees are unsatisfactory.
But, many of these domains are believed to be captured by a small \textit{dynamic Bayesian network} (DBN) on factored state variables.
Therefore, Factored MDPs (FMDPs) were introduced by \cite{og_factored_mdps} to take advantage of such \textit{a priori} knowledge about independence and the structure of the transition function.
%by modeling it as a \textit{dynamic Bayesian network} (DBN) on factored state variables. 
Subsequently, efficient approximate FMDP planners were developed by \cite{guestrin}, and RL in FMDPs was considered by \cite{first_rl_fmdps} assuming access to an efficient FMDP planner.

More recently, \cite{osband} obtained \textit{near-optimal} RL regret bounds in FMDPs assuming access to a stronger planner which returns the optimistic solution to a \textit{family} of FMDPs. 
No polynomial-time algorithm for such a planner is known. Moreover, planning for a single FMDP is intractable \citep{mundhenk2000complexity,lusena2001nonapproximability}, and planning over a family is generally no easier.

%\subsection{Contributions}
Optimization for FMDP learning is difficult in part because when
the factored structure of the unknown transition probabilities are explicitly represented, the resulting problem is a polynomial optimization problem. Even quadratic optimization is NP-hard in general. 
% \citep{delgado}. 
We argue that when given a linear value function with factored structure, the independence of the transition components is unnecessary for obtaining a regret bound, and instead permit potentially correlated transitions on the state variables.
% }
We propose a polynomial time algorithm for RL for this family of \textit{Factored State} MDPs (FSMDPs) with bounded-norm 
% \magenta{
and factored
% } 
linear value functions, assuming an efficient variable elimination order for the induced cost network of the basis is given.\footnote{\label{footnote:variable_elim}%
% For example, we might design such a basis. Alternatively, although optimizing elimination orders is NP-hard, many practical heuristics have been proposed (some even with approximation guarantees). 
We aren't learning a basis or solving for an elimination ordering; the basis and efficient elimination ordering are fixed in advance.
Indeed, moreover, we do not require an optimal elimination ordering, merely that the induced width is adequately small. So the approximation algorithms of
\cite{variable_elimination_1,variable_elimination_2,variable_elimination_3}, could suffice.}
We stress that our algorithm \textit{does not} use an oracle for planning.
\cite{kane2022computational} showed that
% it is necessary to make some form of structural assumption about the basis for RL with linear value function approximation in order to obtain computationally efficient algorithm. 
the general RL problem with linear value function approximation, in which one simply drops our assumption, is intractable. Thus, some assumption is necessary to obtain a polynomial-time algorithm. Recent works (discussed further below) obtained such algorithms by assuming a linear \emph{transition model}, which is restrictive. The conditions for variable elimination on the $V$ function basis, by contrast, are relatively benign and allow us to address some tasks in complex environments.
% , and well-ordered cost network with bounded width.
% In contrast to earlier work,
 % or depend on the diameter of the MDP, which indeed may be disconnected. 
%Furthermore, in contrast to recent works, we \textit{do not} assume the transition model is linear
%Q-function is linear, which according to recent literature, requires the transition function to be linear in order to avoid arbitrarily large Bellman error 
(see \textbf{Sec.}~\ref{sec:not-subsumed}, \textbf{Appx}.~\ref{appx:linear_V_and_linear_Q} for an extended discussion).

% The family of Correlated FMDPs 
% drops the independence assumption amongst state variables contributing to the transition probabilities on the DBN. 
% Such formulation applies to a much boarder range of transition models.
% We also assume that the value function in this family of MDP is linear since we need some form of learnable simple policy, and
% assuming linear value function is less restrictive than assuming linear Q function, which has been used in recent RL literature (see Sec.~\ref{related_works}). 

% {\note Given bounded width cost network with a variable elimination order criterion, we have a polytime algo for planning.} 

Our RL algorithm is based on UCRL-Factored \citep{osband},
which employs an oracle for an \textit{optimistic} planner over a family of FMDPs as a subroutine.
For general FMDPs, it is unclear whether such a planner with polynomial time and theoretical guarantees can exist.
We propose a theoretically grounded and efficient planner 
for 
% {\color{blue} 
FSMDPs
% } 
by modifying the \textit{imprecise} FMDP planner of \cite{delgado}. 
% {\color{blue}
Our formulation has the reward functions $R$ and transition probabilities $P$ take unknown values from bounded convex sets centered on their empirical estimates.
% }

% {\color{red} 
Due to the conditional independence assumption on transition probabilities in FMDP DBNs, 
% } 
and $P$ being variables, the original imprecise FMDP planner formulation of \cite{delgado} inevitably leads to multi-linear programming, which in general is a difficult non-convex problem.
We circumvent this by: 1) removing the conditional independence assumption of the transition model -- hence factored \emph{state} MDPs -- and only computing estimates of the factored marginal transition probabilities which do not need to be consistent; 2) utilizing an \textit{optimistic} formulation as required by UCRL-Factored, which is easier to formulate and solve than the \textit{pessimistic} formulation of \cite{delgado}, which contained a difficult $\min \max$ constraint; and 3) constructing an efficient separation oracle for the program %for the Ellipsoid Algorithm 
by applying the variable elimination procedure proposed by \cite{guestrin}. Note that our planning problem is a convex  program with an exponential number of constraints, 
% {\color{blue} 
which cannot simply be plugged into a standard LP solver to obtain a polynomial-time guarantee.
% }

% 
% {\color{red} We remove the independence assumption of the transition probabilities} in part because we found that it is not inherently necessary to achieve a regret bound similar to UCRL-Factored \cite{osband}. One key theoretical contribution is that we compensate for the lack of independence 
% in our regret analysis.
% %through a different application of Hölder's inequality than 
% Moreover, 
% our modified analysis only requires
% %with the alternative Hölder's inequality, we only need 
% the linear value function to have low norm, and does not depend on the diameter $D$ of the FMDP, which was previously required in this line of work \cite{jaksch, osband}.
% In fact, our algorithm is able to solve an FMDP even if its states are not completely connected.
% Although \cite{xu2020FMDP} circumvent a dependence on $D$ by proposing an algorithm dependent on the \textit{factored span} of an FMDP, they require an oracle even more powerful than a planner.

% To our knowledge, our approach represents the first fully polynomial time scheme with bounded regret for \textit{episodic} RL for the class of correlated FMDPs. 

% \paragraph{Related Work:}
\subsection{Related Works}\label{related_works}
\cite{xu2020FMDP} improve UCRL-Factored for the non-episodic setting by discretizing the confidence sets but still require an oracle planner.
\cite{sra2020FMDP} derive an optimal \textit{minimax} regret bound for episodic RL in FMDPs, but utilize a subroutine \textsc{VI\_Optimism} which performs value iteration to find an optimal policy, iterating over all exponentially many states.
Importantly, our work builds on \cite{jaksch} and \cite{osband} by modifying the underlying structural assumptions to show that exact polynomial-time planning is indeed possible while retaining RL regret bounds similar to their oracle-efficient ones.

Beyond \cite{osband} we also assume that the optimal value function is linear w.r.t. a particular basis  of functions. Linear value functions and approximations have been well studied \citep{linear_value_1, linear_value_3, linear_value_2, osband_linear_v}. The bounds obtained in these works are polynomial in the number of states, however, and the algorithms do not scale to large MDPs that may still have compact FMDPs. \cite{weiszLowerBound} prove an exponential lower bound for linearly-realizable MDPs, however their construction requires an exponential sized action space. We instead assume a polynomial sized action space for tractable planning.

%\subsection{Related Work: Planning}
Imprecise MDPs were first introduced by \cite{MDPIP_introduction} to model transition functions that are imprecisely specified (i.e. could be any function within some convex transition set). Using techniques from \cite{guestrin}, \cite{delgado} proposed a \textit{pessimistic} planner for imprecise FMDPs but could not simultaneously guarantee correctness and efficiency. For the purpose of learning, we instead require (and thus construct) an \textit{optimistic} planner for a family of FMDPs with imprecise transition \textit{and} reward functions. Our setting is similar to the Bounded MDPs introduced by \cite{BoundedMDPs}, but with an exponential-sized state space, additional linear structure, and a less strict requirement on ``well-formed transition functions".

There is also a line of work on simultaneous FMDP structure and reinforcement learning \citep{structure_learning_1, structure_learning_2}. We instead assume that such structure is given as input in the RL problem.

Other assumptions for RL with large state-space such as low Bellman rank  \citep{jiang2017contextual}, Bellman Eluder \citep{jin2021bellman}, and bi-linear class \citep{du2021bilinear} are structural conditions that permit sample-efficient RL.
However, their algorithms all use the optimization algorithm OLIVE of \cite{jiang2017contextual}, which uses an optimistic planner that is not efficient in general.

% The only previous assumptions we are aware of that permit a provably efficient planner are either some form of block MDPs or linear MDPs.
Block MDPs \citep{du2019provably} permit a provably efficient planner, but are only solved efficiently when the number of blocks is small, i.e., there is essentially a small latent state space. Obviously, this substantially restricts the possible richness of the environment.
Computationally efficient algorithms were also obtained by
\cite{jordan_linear_rl} assuming linear transitions and rewards 
% in the non-discounted case.
in RL with finite episodes, and by 
\cite{yang_linear_q} in the discounted setting with a linear transition model.
(Both show these assumptions imply the optimal $Q$-function is linear.)
Our work instead assumes a linear state value ($V$) function, which is not captured by linear transition models (\textbf{Sec}.~\ref{sec:not-subsumed}).
\cite{wang2020GeneralValue} instead focus on RL with general Q-function approximation, with bounds parameterized by the \textit{Eluder} dimension \citep{russo2013eluder}, which may be large in our setting (\textbf{Sec}.~\ref{sec:not-subsumed} again).

\section{Preliminaries}
% Our work is based on the confidence set approach from \cite{jaksch} and \cite{osband}, both of whom consider RL in a non-discounted, cumulative episodic reward setting introduced by \cite{cumulative_reward_rl}. 
Our work considers RL in a non-discounted, cumulative episodic reward setting introduced by \cite{cumulative_reward_rl}. 
Consequently, the value function may take different values at the same state at different points in the time horizon $\tau$. Therefore, any approach to RL in this setting must solve for a different value function at each time step. 

Let $M = (\mathcal{S}, \mathcal{A}, R^M, P^M, \tau, \rho)$ be a finite horizon MDP.
Each episode is a run of the MDP $M$ with the finite time horizon $\tau$.
$R^M: \mathcal{S} \times \mathcal{A} \to \mathbb{R}$ is a reward distribution from $(s,a)$ pairs, $P^M(s'|s,a)$ is the transition probability over $\mathcal{S}$ from $s \in \mathcal{S}, a \in \mathcal{A}$, and $\rho$ the initial distribution over $\mathcal{S}$.

% A deterministic policy $\mu$ is a function mapping each state $s \in \mathcal{S}$ to an action $a \in \mathcal{A}$. 
% For an MDP $M$ and policy $\mu$, we define the discounted value function as follows:
% $%\begin{equation*}
%     V_{\mu}^M(s) \coloneqq \mathbb{E}_{M, \mu} [ \sum_{i=1}^\tau \gamma^i \overline{R}^M (s_i, a_i) \mid s_1=s ]
% $%\end{equation*}
% where $\overline{R}^M (s, a)$ is the expected reward for taking action $a$ in state $s$.
% The subscripts of $\mathbb{E}$ denote that $a_i = \mu(s_i)$ and $s_{i+1} \sim P^M(\cdot | s_i, a_i)$ for each $i$. 
% A policy $\mu$ is optimal if $V_{\mu}^M(s) = \max_{\mu'} V_{\mu'}^M(s)$ for all $s \in \mathcal{S}$. 
% We denote the optimal policy for an MDP $M$ by $\mu^M$.
A deterministic policy $\mu$ is a function mapping each state $s \in \mathcal{S}$ to an action $a \in \mathcal{A}$. 
For an MDP $M$ and policy $\mu$, we define a value function as:
$
    V_{\mu, \ell}^M(s) \coloneqq \mathbb{E}_{M, \mu} [ \sum_{\ell'=\ell}^\tau \overline{R}^M (s_{\ell'}, a_{\ell'}) \mid s_\ell=s ]
$ for each step $\ell = 1,\dots, \tau$, where $\overline{R}^M (s, a)$ is the expected reward for taking action $a$ in state $s$.
The subscripts of $\mathbb{E}$ denote that $a_\ell = \mu(s_\ell)$ and $s_{\ell+1} \sim P^M(\cdot | s_\ell, a_\ell)$ for each $\ell$. 
A policy $\mu$ is optimal if $V_{\mu, \ell}^M(s) = \max_{\mu'} V_{\mu', \ell}^M(s)$ for all $s \in \mathcal{S}$. 
Let $\mu^M$ denote an optimal policy for MDP $M$.
% \begin{definition}
% \label{def:bellman}
% The Bellman operator $\mathcal{T}_{\mu^M, i}$ for the MDP $M$, policy $\mu$, and value function $V: \mathcal{S} \to \mathbb{R}$, is defined by
% %\begin{align*}
% %    \mathcal{T}_\mu^M V(s) \coloneqq &
%     $
%     \mathcal{T}_{\mu^M,i} V_i(s) = 
%     \overline{R}^M(s, \mu(s))
%      + \sum_{s' \in \mathcal{S}} P^M(s' | s, \mu(s)) V_{i+1}(s').
%      $
% %\end{align*}
% % where $Z_j^h$ is the scope of the $j$th basis function $h_j$. 
% \end{definition}
The RL agent interacts with some latent $M^*$ in the environment over episodes, where each episode begins at $t_k = (k-1)\tau + 1, k=1,2,...$. At time step $t$, the agent selects an action $a_t$, observes a scalar reward $r_t$, then transitions to $s_{t+1}$. 
Let $H_t = (s_1, a_1, r_1, \dots, s_{t-1}, a_{t-1}, r_{t-1})$ be the history of observed transitions prior to time $t$. 
An RL algorithm outputs
a sequence of functions $\{\pi_k \mid k=1, 2, \dots\}$, each mapping $H_{t_k}$ to a probability distribution $\pi_k (H_{t_k})$ over policies which the agent will employ in episode $k$. The regret incurred is defined as
%\begin{equation*}
$    \text{Regret}(T, \pi, M^*) \coloneqq \sum_{k=1}^{\lceil T/ \tau \rceil} \Delta_k
$
%\end{equation*}
where $\Delta_k$ is the regret over the $k$th episode:%\useshortskip
\begin{equation}
    \label{eq:regret_episode_k}\small
    % \Delta_k \coloneqq 
    % \mathbb{E}_{s \sim \rho} \left[
    % V_{\mu^*, 1}^{M^*}(s_{t_k + 1}) - V_{\mu_k, 1}^{M^*}(s_{t_k + 1}) \right]
    \Delta_k \coloneqq 
    \mathbb{E}_{s \sim \rho} \left[
    V_{\mu^*, 1}^{M^*}(s) - V_{\mu_k, 1}^{M^*}(s) \right]
\end{equation}
with $\mu^* = \mu^{M^*}$, $\mu_k \sim \pi_k (H_{t_k})$. 

\subsection{Factored State MDPs and Structured Linear Value Functions}
We are interested in MDPs with possibly exponential sized state spaces but containing \textit{factored} structure.
\begin{definition}
\label{def:scope_variable}
Let $\mathcal{X} = \mathcal{X}_1 \times \cdots \times \mathcal{X}_n$. For any subset of indices $Z \subseteq [n]$, the \textit{scope operation} of a set is defined as $\mathcal{X}[Z] \coloneqq \underset{i \in Z}{\bigotimes} \mathcal{X}_i$. For any $x \in \mathcal{X}$ we can define the scoped variable $x[Z] \in \mathcal{X}[Z]$ to be the values of the variables $x_i \in \mathcal{X}_i$ with indices $i \in Z$. 
% Similarly, we will use lower case letters to denote specific values.
\end{definition}
%By slight abuse of notation, 
For simplicity of notation we will also write $\mathcal{X} = \mathcal{S} \times \mathcal{A} = \mathcal{X}_1 \times \dots \times \mathcal{X}_n $ in RL, where the action space $\mathcal{A}$ has constant cardinality but $\mathcal{S}$ can be exponentially large.

We assume the transition function in the environment is defined as follows with respect to the scopes of state variables:
\begin{definition}\label{def:factored_transition_func}
A \emph{Factored State MDP} (FSMDP) is an MDP defined by a set of marginal transition probabilities $\mathcal{P} = \left\{ P_i( s'[Z^p_i] | s[\text{Pa}(Z^p_i)], a) \right\}_i$ such that the probability of transitioning to $s'[Z^p_i]$ is independent of state variables outside the scope $s[\text{Pa}(Z^p_i)] \subseteq \mathcal{S}$, i.e.,
where $\text{Pa}(Z^p_i) \subseteq [n]$ denotes the variables within $\mathcal{S}$ that $Z^p_i$ depends on in the transition. 
\end{definition}

We assume that the environment has the same reward structure as \cite{osband}:
\begin{definition}
\label{def:factored_reward_func}
The reward function class $\mathcal{R}$ is factored over
$\mathcal{S} \times \mathcal{A}$
% $\mathcal{X} = \mathcal{X}_1 \times \cdots \times \mathcal{X}_n$
with scopes $Z_1^R, \dots, Z_l^R \subseteq [n]$ iff for all $R\in\mathcal{R}, x\in \mathcal{X}$ there are functions $\{ R_i \in \mathcal{P}_{\mathcal{X}[Z_i^R], \mathbb{R}}^{C, \sigma}\}_{i=1}^l$ such that
%\begin{equation*}
    $\mathbb{E}[r] = \sum_{i=1}^l \mathbb{E}[r_i]$
%\end{equation*}
where $r \sim R(x)$ is equal to $\sum_{i=1}^l r_i$ with each $r_i \sim R_i(x[Z_i^R])$ \textit{individually observed}. 
% Furthermore, we assume $|Z_i^R| \in O(\log m)$ for all $i \in [l]$.
Here $\mathcal{P}_{\mathcal{X}[Z], \mathbb{R}}^{C, \sigma}$ denotes the set of functions mapping $\mathcal{X}[Z]$ to $\sigma$-subgaussian probability measures over the measure space $(\mathbb{R}, \mathcal{B}(\mathbb{R}))$ with mean in $[0, C]$ and Borel $\sigma$-algebra $\mathcal{B}(\mathbb{R})$.
\end{definition} 

For tractable learning, we assume that there is a factored linear value function class:
% $V_\ell$ at each step $\ell$ has a compact linear structure.
% Similar to the reward function, it relies on only a small subset of state variables on each component.
\begin{definition}
\label{def:factored_value} % TODO: Modify to fit later definitions (probably)
The value function class $\mathcal{V}$ is linear and \textit{factored} over $\mathcal{S} = \mathcal{S}_1 \times \cdots \times \mathcal{S}_m$ with scopes $Z_1^h, \dots, Z_\phi^h \subseteq [m]$ 
iff 
there exists a set basis functions $h_j:\mathcal{S}[Z_j^h] \mapsto \mathbb{R}, j \in [\phi]$,  
such that 
for any function 
$V \in \mathcal{V}$ 
% $V^{M}_{\mu^M} \in \mathcal{V}$ 
% $$\in \mathcal{V}$ 
% {\color{blue} 
% (with associated optimal policy $\mu^M$ for MDP $M$)
% }
we have
% for basis functions $h_j:\mathcal{S}[Z_j^h] \mapsto \mathbb{R}$ 
% and associated $w^_j \in \mathbb{R}$ for $j=1,\dots, \phi$,
$ 
     V(s) = \sum_{j=1}^\phi w_j h_j(s[Z_j^h]) \text{ for all } s \in \mathcal{S}, 
$ 
for some weight vector $\mathbf{w} \in \mathbb{R}^{\phi}$. 
% such that $\| w\|_1 \leq W$. 
\end{definition}

We assume the true value functions at each step in an episode are linear and factored:
$V^{M^*}_{\mu^*, \ell} = \sum_{j=1}^\phi w_j^{*(\ell)} h_j(s[Z_j^h]), \text{ for all } s \in \mathcal{S}, \ell = 1,\dots, \tau$. 

We assume that the scopes $\{ Z^p_i \}_i$ and $\{Z^h_j\}_j$ are the same in our environment.
In that case,
%\textbf{Def.}~\ref{def:factored_value} and \textbf{Def.}~\ref{def:factored_transition_func} are compatible.
%This is because 
the second term of the Bellman operator %in \textbf{Def}.~\ref{def:bellman} 
simplifies to the following with a factored linear $V$ (similar to \cite{koller}).
% $$ 
% \sum_{j=1}^\phi w_j \sum_{\hat{s}' \in \text{Val}(Z_j^h)}  h_j(\hat{s}') P(\hat{s}'| s[\text{Pa}(Z_j^h)], a)
% $$
\begin{align}\small
    &\sum_{s' \in \mathcal{S}} P(s' | s, a) V(s') \nonumber \\
   =& 
   \sum_{s' \in \mathcal{S}} P(s' | s, a) \sum_{j=1}^\phi w_j h_j (s'[Z_j^h])
   \nonumber \\
   =& \sum_{j=1}^\phi w_j \sum_{\mathclap{\hat{s}' \in \text{Val}(Z_j^h)}}  h_j(\hat{s}') 
   \sum_{\mathclap{ \bar{s}' \in \text{Val}(\overline{Z}_j^h)}}P(\hat{s}', \bar{s}' | s, a) \nonumber \\
   =& \sum_{j=1}^\phi w_j \sum_{\mathclap{\hat{s}' \in \text{Val}(Z_j^h)}}  h_j(\hat{s}') P(\hat{s}'| s, a) \nonumber \\
   =& \sum_{j=1}^\phi w_j \sum_{\mathclap{\hat{s}' \in \text{Val}(Z_j^h)}}  h_j(\hat{s}') P_j(\hat{s}'| s[\text{Pa}(Z_j^h)], a). \label{eq:backprojection}
\end{align}
% Val$(\cdot)$ is the \textit{assignment operator}. 
Val$(Z_j^h)$ is the \textit{set} of all assignments to state variables in $Z_j^h$, and Val$(\overline{Z}_j^h)$ to variables $\notin Z_j^h$, e.g., if the scope $\mathcal{S}[Z_j^h]$ has three binary state variables, then we would have Val$(Z_j^h) = \{0, 1\}^3$.
Here we split $s'$ into $\hat{s}'$ and $\bar{s}'$.
We first marginalized out $\bar{s}' \in \text{Val}(\overline{Z}_j^h)$, the parts of the state $s$ which are outside of the scope of the $j$th basis function (\textbf{Def}.~\ref{def:factored_value}), then condition $P$ on the variables of $s$ that occur in the parents w.r.t.\ $a \in \mathcal{A}$ of the scope $Z_j^h$.
Thus we only need to keep track of the marginal probabilities $P_j( s[Z^h_j] | s'[\text{Pa}(Z^h_j)], a)$ instead of $P(s| s', a)$ as in standard MDPs. 
% See \textbf{Appx.}~\ref{appx:decompisition_derivation} for derivation.
% In tabular RL, we keep track of all $P(s'| s, a)$ parameters, but that's impractical in FMDP. 

% {\color{blue}
% In regular FMDP, we keep track of the set of factored transition probability parameters $P_i( s_i | x[Z^P_i])$, whose cardinality is exponentially smaller than the set of $P(s'| s, a)$ in tabular RL. 
% In marginalized FMDP, we instead use a collection $\mathcal{P}$ of marginal distributions $P( s[Z^h_j] | s[\text{Pa}(Z_j^h)], a)$ over associated domains Val$(Z_j^h)$ to fully describe the transition dynamics of the FMDP. 
% % In the LP algorithm, we will directly use each $P( s[Z^h_j] | s[\text{Pa}(Z_j^h)], a)$ as a single term instead of a multi-linear term $P( s[Z^h_j] | s[\text{Pa}(Z_j^h)], a)  = \prod_{i \in Z^h_j} P_i( s_i | x[Z^h_j])$ to avoid multi-linear programming as in \cite{delgado}. 
% % Note that $P_i( s_i | x[Z^h_j)$ within $Z^h_j$ do not have to be independent with each other anymore if we only care about the marginals. 
% }

% where each distribution depends only on an assignment of variables to the parents of all the variables in scope $Z_j^h$ as well as some action $\mu(s) \in \mathcal{A}$. 

% Instead of keeping track of all the probabilities $P(s| s', a)$, we learn the marginal probabilities $ P( s[Z^h_j] | s'[\text{Pa}(Z^h_j)], a)$. 
The number of distinct marginals required to represent the FSMDP is bounded:
\begin{align}
\label{eq:poly_big_oh}\small
    |\mathcal{P}| &\leq \sum_{j=1}^\phi |\mathcal{A}| |\text{Val}(\mathcal{S}_k)|^{|\text{Pa}(Z_j^h)|}\\
    %&
    &\leq O(\poly(m) |\text{Val}(\mathcal{S}_k)|^{\zeta} \phi)
    % \leq O(\poly(m))
\end{align}
where $\zeta \geq |\text{Pa}(Z_j^h)|$ is a scope size bound, and $|\text{Val}(\mathcal{S}_k)|$ denotes the number of values a state variable can take. 

\begin{remark}
Our setting captures interesting environments. Consider for example a gridworld, in which there is a penalty for colliding with other, randomly moving objects. If there is a safe policy, the optimal $V\equiv 0$ (and is thus linear), and the local movement ensures a compact DBN. Yet, the presence/absence of objects is \emph{not} independent across positions.
Indeed, since the location of objects in the gridworld is mutually exclusive over the grid, the factors are negatively correlated. Therefore such environments cannot be captured by the usual FMDPs (with independent factors) but they are captured by FSMDPs.
\end{remark}
\begin{remark}
FSMDPs subsume regular FMDPs in the linear value function case, as our transition marginals can express conditionally-independent \textit{and} non-conditionally-independent transition functions.
\end{remark}

% \section{LINEAR $V$ FUNCTIONS WERE NOT PREVIOUSLY ADDRESSED}
\section{Linear $V$ Functions Were Not Previously Addressed}
\label{sec:not-subsumed}
% {\color{blue}
% We only need to point out that other linear Q paper restricts transition function is linear. Maybe this section is a bit distracting and belong to appendix.
% }

% Our model assumption contains factored linear $V$ function (state value function) in a \emph{structured} MDP with exponentially large space.
We stress that we learn MDPs that weren't addressed by prior work.
As we discussed in  \textbf{Sec}.~\ref{related_works} recent literature has mostly focused on \emph{sample efficiency} in RL problems with large state space (whose regret bound does not depend on the state space). However, they usually involve a planner that is potentially intractable. 
% Almost all existing polynomial time complexity algorithms that can tackle exponentially large state spaces (whose regret bound does not depend on the state space) 
Other than Block MDPs, whose difference from our problem class is clearer, 
linear transition function is a common assumption that permits polynomial \emph{time complexity} in large state space  \citep{jordan_linear_rl, yang_linear_q}. 
It's obvious that 
this highly restricts the learnable environments.
% In contrast, we only require the $V$ function to be linear, 
% % To motivate our use of factored linear $V$ function, 
% where we can show that it's possible to have a compact linear $V$ function but 
% % a very complicated $Q$ function (\textbf{Appx.}~\ref{appx:linear_V_and_linear_Q}):
% not a linear transition function. 
Indeed, even if the $V$ function is linear, the transition function can be nonlinear (please see \textbf{Appx.~}\ref{appx:linear_V_and_linear_Q} for details).

\begin{proposition}
\label{prop:V_vs_Q_in_main_paper}
Let a state-action ($Q$-function) basis $\{h_1(s,a), .., h_\phi(s,a)\}$ be given such that $\phi < N=2^m$. Then there is an MDP family $\mathcal{M}$ on $N$ states ($m$ binary factors) for which the optimal $Q$-function cannot be expressed as a linear combination of these basis functions with high probability ($1-2^{-N+\phi} \geq 1/2$) for any MDP $M \in \mathcal{M}$, whereas every MDP $M \in \mathcal{M}$ has a compact, optimal linear $V$ function representation for \textit{any} given basis set of state feature functions.
\end{proposition}

\cite{jordan_linear_rl, yang_linear_q} proved that linear transition functions imply linear $Q$ functions. So, contrapositively: 
\begin{corollary}
There exists an MDP with a linear $V$ function but not a linear transition function. 
% Therefore in our example, the transition function is not linear.
\end{corollary}

Moreover, in addition to not having a nice linear form, the Q-function in our example can also have a high Eluder dimension, since the MDP is a random environment when one of the unsafe actions is chosen---specifically, fixing a sequence of actions is not informative about the effect of subsequent actions until/unless the process revisits a state, which is unlikely in our exponential state space. Indeed, \citet{russo2013eluder} gave lower bounds on the Eluder dimension that carry over to our example.

% Therefore in our example, the transition function is not linear either. 
% On the other hand, 
% it's common for these linear $Q$ function based works to require transition function to be linear  \citep{jordan_linear_rl, yang_linear_q} as well.
% Because value iteration based algorithms with linear $Q$ function require Bellman Error to be zero to avoid divergence, 
% % This either needs to be explicitly assumed for the entire linear $Q$ function class \citep{zanette2020learning}, or it
% and requiring transition function to be linear helps achive this. 

% In contrast, we do not need transition function to be linear.
% Moreover, since our algorithm is not value iteration based and our problem has a finite-episode, the restriction does not apply. 
% {\color{blue} Do we need to justify that Bellman Error is not inherent in the DP of Bellman Operator???}

%\section{RL Algorithm framework}
\section{Algorithm}
\label{sec:algorithm_section}
% {\color{gray} 
% For simplicity we first illustrate our algorithm in the discounted episodic RL setting with the discounted Bellman Operator:
%     $
%     \mathcal{T}_\mu^M V(s) = 
%     \overline{R}^M(s, \mu(s))
%      + \gamma \sum_{s' \in \mathcal{S}} P^M(s' | s, \mu(s)) V(s').
%      $
% Each episode is finite but ``long enough'' for us to consider the Bellman operator time-invariant. Therefore, there is a single value function $V^*$ we need to estimate.
% Afterwards, we discuss how to extend our result to the finite episode setting with a potentially different value function $V_h$ for each step $h \in [\tau]$.
% }

Our proposed algorithm modifies UCRL-Factored \citep{osband}, keeping track of confidence sets around each $R^M_i$ and marginal distribution $P^M(\cdot | s[\text{Pa}(Z_j^h)], a)$, where the true $R^{M^*}, P^{M^*}$ reside w.h.p. We use the definition of \cite{osband}:
The confidence set at time $t$ is centered at an empirical estimate $\hat{f}_t \in \mathcal{M}_{\mathcal{X}, \mathcal{Y}}$ defined by
$
%\begin{equation*}
    \hat{f}_t(x) = \frac{1}{n_t(x)} \sum_{\tau < t:x_\tau = x} \delta_{y_\tau},
$
%\end{equation*}
where $n_t(x)$ counts the number of occurrences of $x$ in $(x_1, \dots, x_{t-1})$ and $\delta_{y_t}$ is the probability mass function over $\mathcal{Y}$ which assigns all probability to outcome $y_t$.
Our sequence of confidence sets depends on a choice of norm $|| \cdot ||$ and a non-decreasing sequence $\{d_t : t \in \mathbb{N} \}$. For each $t$, the confidence set $\mathcal{F}_t = \mathcal{F}_t ( || \cdot ||, x_1^{t-1}, d_t)$ is defined as:
\begin{equation*}%\begin{align*}
\small
    \bigg\{ 
        f \in \mathcal{F} \bigg|
        ||(f - \hat{f}_t)(x_i)|| \leq \sqrt{\frac{d_t}{n_t(x_i)}}\ \forall i \in [t-1]
    \bigg\}.
\end{equation*}%\end{align*}

\begin{algorithm}[tb]
\caption{UCRL-Factored for FSMDP}
\label{algo:ucrl_factored}
\begin{algorithmic}
\FOR{episode $k=1\dots K$}
    \STATE $d_t^{R_i} = 4\sigma^2 \log(4l|\mathcal{X}[Z_i^R]|k / \delta)$ for $i=1\dots l$ 
    \STATE $d_{t_k}^{P_j} = 2 |\text{Val}(Z_j^h)|\log(2)$
    %\STATE \myquad[2.5] 
    $- 2 \log(\delta / (2 N |\text{Pa}[Z_j^h]| k^2 ))$ for $j=1 \dots N$ 
    \STATE $\mathcal{M}_k = \{ M | \overline{R}_i \in \mathcal{R}_t^i(d_t^{R_i}), P_j \in \mathcal{P}_t^{j}(d_t^{P_{j}})\ \forall i, j \}$ 
    \STATE $\mu_k = \operatorname{OptimisticPlanner}(\mathcal{M}_k, \epsilon=\sqrt{1/k})$ 
    \STATE sample initial state variables $s_1^1, \dots, s_1^m$ 
    \FOR{timestep $t = 1\dots \tau$}
    \STATE sample and apply $a_t = \mu_k(s_t)$ 
    \STATE    observe $r_t^1, \dots, r_t^l$ and $s_{t+1}^1, \dots, s_{t+1}^m$ 
    
    \ENDFOR
\ENDFOR
\end{algorithmic}
\end{algorithm}

% \begin{algorithm}
% \label{algo:ucrl_factored}
% \For{episode $k=1\dots K$}{
%     $d_t^{R_i} = 4\sigma^2 \log(4l|\mathcal{X}[Z_i^R]|k / \delta)$ for $i=1\dots l$ \\
%     $d_{t_k}^{P_j} = 2 |\text{Val}(Z_j^h)|\log(2)$\\ 
%     \myquad[2.5] $- 2 \log(\delta / (2 N |\text{Pa}[Z_j^h]| k^2 ))$ for $j=1 \dots N$ \\
%     $\mathcal{M}_k = \{ M | \overline{R}_i \in \mathcal{R}_t^i(d_t^{R_i}), 
%     P_j \in \mathcal{P}_t^{j}(d_t^{P_{j}})\ \forall i, j \}$ \\
%     $\mu_k = \operatorname{OptimisticPlanner}(\mathcal{M}_k, \epsilon=\sqrt{1/k})$ \\
%     sample initial state variables $s_1^1, \dots, s_1^m$ \\
%     \For{timestep $t = 1\dots \tau$}{
%         sample and apply $a_t = \mu_k(s_t)$ \\
%         observe $r_t^1, \dots, r_t^l$ and $s_{t+1}^1, \dots, s_{t+1}^m$ \\
%     }
% }
% \caption{UCRL-Factored for FSMDPs}
% \end{algorithm}

We write $\mathcal{R}_t^i(d_t^{R_i})$ as shorthand for the reward confidence set $\mathcal{R}_t^i(|\mathbb{E}[\cdot]|, x_1^{t-1}[Z_i^R], d_t^{R_j})$ and 
$\mathcal{P}_t^{j}(d_t^{P_{j}})$  for a vector of confidence sets
% $\underset{a \in \mathcal{A}}{\bigotimes} \mathcal{P}_t^{j,a}(\| \cdot \|_1, (s_1^{t-1}[\text{Pa}(Z_j^h)], a_{1}^{t-1}), d_t^{P_{j, a}})$,
$\mathcal{P}_t^{j,a}(\| \cdot \|_1, (s_1^{t-1}[\text{Pa}(Z_j^h)], a_{1}^{t-1}), d_t^{P_{j, a}})$,
% $\forall a \in \mathcal{A}$ 
% the set of $|\mathcal{A}|$ confidence sets 
over ($j$th marginal, action $a$) pairs.

% %Intuitively, 
% UCRL-Factored iteratively refines confidence sets containing the true, factored reward and transition functions w.h.p., while using an oracle planner over these sets to obtain a policy for each episode of learning.
Let $N = |\mathcal{P}|$ be the number of transition function marginals in \eqref{eq:poly_big_oh}. \textbf{Alg.}~\ref{algo:ucrl_factored} gives our full RL algorithm which modifies UCRL-Factored by changing the number and choice of confidence set sequences, and using the $\operatorname{OptimisticPlanner}$ we propose %in \textbf{Sec}.~\ref{sec:planner} 
instead of an oracle.

\begin{algorithm}[tb]
\caption{$\operatorname{OptimisticPlanner}$}
\label{algo:opt_planner}
\begin{algorithmic}
\STATE $\mathbf{w} = \mathbf{w}_0$ \hfill // centroid of the initial large ellipsoid 
\STATE $M_R \leftarrow$ optimistic rewards $\overline{R}_i(z)$ with \eqref{eq:reward_fn_opt}
\STATE $M_P \leftarrow$ optimistic transition marginals (\textbf{Alg.~}\ref{algo:trans_fn_opt}, \textbf{Apx.}~\ref{appx:variable_elimination})
\STATE $\Omega \leftarrow$ Simplify constraints of \eqref{fig:lp_factored} with variable elimination \textbf{Alg.~}\ref{algo:simplify_oracle_obj} and computed $M_R, M_P$ 
    \WHILE{$\mathbf{w}$ does not satisfy constraints $\Omega$}
    \STATE Use tightness to construct cutting-plane (\textbf{Thm.}~\ref{them:oracle_poly})
    \STATE $\mathbf{w}$ $\leftarrow$ new ellipsoid centroid within cutting-plane
    \STATE $\Omega$ = Simplify constraints of \eqref{fig:lp_factored} with variable elimination \textbf{Alg.~}\ref{algo:simplify_oracle_obj} and computed $M_R, M_P$
    \ENDWHILE
\end{algorithmic}
\end{algorithm}

We formulate our MDP planning task as an LP solving for the optimal value function $V^*(s)$ over each state $s$.
% (\textbf{Alg.}~\ref{algo:opt_planner}). 
% After computing $V^*$, we obtain the optimal policy by greedily taking actions.
By using the fact that $V_\ell$ and $V_{\ell+1}$ are related through the Bellman operator $V_\ell(s) = \max_a \{ R(s,a) + \sum_{s'} P(s'|s,a) V_{\ell+1}(s')\}$,
% which is equivalent to 
% \begin{align}
%  & \min_{V_i} \sum_s {V_i}(s) \\  \nonumber 
%  s.t. & \  V_i(s) \geq R(s,a) + \sum_{s'} P(s'|s,a) V_{i+1}(s')
%  \ \forall s \in \mathcal{S}, a \in \mathcal{A} \nonumber
% \end{align}
and inductively 
% following the same argument as \textbf{Lemma 1} of \cite{delgado}, 
applying the tightness of the LP at its optimum,  
we can show that planning with multiple $V_i$'s is equivalent to the following linear programming problem 
(Please see \textbf{Appx.}~\ref{appx:planning_extension} for details):
\begin{align}\label{eq:multi_V_LP}\small
& \min_{V_1} \sum_{s} V_1 (s) 
\\
s.t. \ \ \ &
V_{\ell}(s) \geq R(s, a) + \sum_{s'} P(s' | s, a) V_{\ell+1} (s'), \\ \nonumber 
& \forall s \in \mathcal{S}, a \in \mathcal{A}, \ \ \ \ell = 1, \dots, \tau, 
\nonumber \\ 
& V_{\tau+1} (s) = 0, \ \ \  \forall s \in \mathcal{S}.  \nonumber
\end{align}
% Note that even though the optimization objective is sum over all $s$, the optimal solution produces the optimal $V_1(s)$ for each $s$. Because if one of $V_1(s)$ is sub-optimal, $\sum_{s} V_1 (s)$ will not be optimal. 

\begin{remark}
We stress that in contrast to prior works, we are not using value iteration, but rather solving a convex program for the $V$ function. Therefore, we don't run into the problem of whether or not the iterates of Bellman operator remain close to the subspace spanned by the basis functions. 
\end{remark}

The seminal work by \cite{guestrin} showed that the
{
% \color{blue}  
Approximate Linear Programming
}
formulation for planning in an FMDP gives the optimal value function $V^*$ iff $V^*$ lies within the subspace spanned by the chosen basis. 
% Their optimization problem is defined as minimizing the objective $\sum_{s} \sum_{j=0}^\phi w^{(1)}_j h_j(s)$ subject to $|\mathcal{S} \times \mathcal{A}|$ constraints. %%\useshortskip
Using the linear value function assumption, each of the inequality constraints can be written in the following form, where $w^{(\ell)}_j$ denotes the coefficient of the basis function $h_j$ in the linear representation of $V_\ell$.  
\begin{equation}\label{eq:guestrin_ALP}\small
\sum_{j=0}^\phi\! w^{(\ell)}_j h_j(s) \geq R(s,a) \!+\!\! \sum_{s'\in \mathcal{S}}\! P(s'|s,a) \!\sum_{j=0}^\phi\! w^{(\ell+1)}_j h_j(s').
\end{equation}
We also include a constant basis function $h_0$ to ensure the LP is feasible \citep{guestrin}. 
% for \textit{each} $(s,a)$ pair. $R$ is a deterministic function.

\begin{figure*}
\centering
\begin{align*}
%\textrm{s.t.} \quad &
% \textrm{s.t.} \quad 
&\forall (s, a, \ell) \in  \mathcal{S}\times \mathcal{A}\times [\tau]\  \sum_{j=0}^\phi w^{(\ell)}_j h_j(s) 
\geq \sum_{i=1}^l \overline{R}_i(s,a) 
+  \sum_{j=0}^\phi \sum_{\hat{s}' \in \text{Val}(Z_j^h)} \!\!\!\!\!\!
    w^{(\ell+1)}_j h_j(\hat{s}') 
    P^{(\ell+1)}_j(\hat{s}'| s[\text{Pa}(Z_j^h)], a)
    \ \  (w_j^{(\tau+1)}\!=0)\\
& (\overline{R}_i)_{i=1}^{l}, (P^{(\ell+1)}_j(\cdot | s[\text{Pa}(Z_j^h)], a) )_{j=1}^{\phi}
= \underset{\mathclap{
                \substack{
                \tilde{\overline{R}}_i \in \mathcal{R}_t^i(d_t^{R_i}) \\
                \tilde{P}_j(\cdot | s[\text{Pa}(Z_j^h)], a) \in \mathcal{P}_t^{j}(d_t^{P_{j}})
                }}
           }{\argmax}
\ \ \ 
\sum_{i=1}^l \tilde{\overline{R}}_i(s,a)
 + \sum_{j=0}^\phi \sum_{\hat{s}' \in \text{Val}(Z_j^h)} \!\!\!\!\!\!
    w^{(\ell+1)}_j h_j(\hat{s}') 
    \tilde{P}^{(\ell+1)}_j(\hat{s}'| s[\text{Pa}(Z_j^h)], a)
\end{align*}
\caption{Constraints for the OptimisticPlanner optimization problem. The objective is $\underset{w}{\min} \sum_{s} \sum_{j=0}^\phi w^{(1)}_j h_j(s)$.
% Note that we can drop the ``argmax'' bilevel constraints and only add the convex set constraint of $P, R$ being in the convex sets.
} 
\label{fig:lp_factored}
\end{figure*}

In \textbf{Alg.}~\ref{algo:ucrl_factored}, the reward distribution and transition functions are learned by successively updating the corresponding confidence sets for each reward component function $\mathcal{R}_t^i(d_t^{R_i})$ and transition function marginal $\mathcal{P}_t^{j}(d_t^{P_{j}})$.
Combining the formulations of \cite{guestrin} and \cite{delgado}, we obtain the imprecise LP formulation \textbf{Fig}.~\ref{fig:lp_factored} for FSMDP, where $R$ and $P$ are defined over bounded convex sets centered on an empirical estimate of the reward and transition functions (see \textbf{Appx}~\ref{appx:opt_problem_formulation}).
The $\argmax$ in \textbf{Fig}.~\ref{fig:lp_factored} specifies an \textit{optimistic} solution, which guarantees that the reward and transition function are set to the best possible value within their respective confidence sets. In this formulation, the variables are: the linear weights $w$, rewards $R$, and transition probabilities $P$.

% We also include a constant basis function $h_0$, which is required for the LP to be feasible \citep{guestrin}. 
% The number of constraints in this program is a function of the exponential number of states. 
% % {\color{blue}
% Correctness of this formulation can be shown through a simple modification to Thm.~1 of \cite{delgado}.

% % }

Although \textbf{Fig}.~\ref{fig:lp_factored} is presented as a non-trivial bilevel program, we argue that we can construct an efficient separation oracle to solve it with 
an algorithm such as the Ellipsoid method \citep{ellipsoid_book} (or a more efficient equivalent \citep{faster_cutting_plane})
%the Ellipsoid -- or in fact, any cutting plane (\textbf{Rmk.~}\ref{rmk:fasterThanEllipsoid}) -- method 
in polynomial time. 
We accomplish this by removing the bilevel constraints and adding a polynomial number of linear constraints describing all possible variations of $R$ and $P$ within their confidence sets: 
while the product of $w$ and $P$ seemed to introduce nonlinear terms in the formulation, we treat the possible values of $P$ as a family of constraints. Indeed, the $\argmax$'s for $R$, $P$ of \eqref{eq:oracle_obj} are the largest of the RHS for the family of constraints we generate in \textbf{Fig}.~\ref{fig:lp_factored}, so the two programs are equivalent.
This reduces the problem to an LP over the exponential sized state space---importantly, since we no longer seek to represent a factorization of $P$, we are able to avoid the terms $P( s[Z^h_j] | s[\text{Pa}(Z_j^h)], a)  = \prod_{i \in Z^h_j} P( s_i | x[Z^h_j])$. 
which exist in \cite{delgado}.
Our problem is thus linear rather than multi-linear. 
% ensuring our problem is not a multilinear programming problem and can thus be solved efficiently.

In an Ellipsoid based algorithm, at each step we fix some $w$, and use a \emph{provably efficient} algorithm implementing a ``separation oracle'' that either identifies the feasibility of the LP with the given $w$ or finds a violated constraint. If it is infeasible, then we find a new $w$ satisfying the additional constraint, and so on.

We note that the arg max computation (which also appears in \eqref{eq:oracle_obj}) for $P_j$ depends (only) on $\{\text{sign}(w^{(\ell)}_j )\}_j$. We further relax each $P_j$ to a set $\{P^{(\ell)}_j\}_{\ell\in [\tau]}$, one for each step. %$w^{(\ell)}_j$.
All $P^{(\ell)}_j$ are still constrained by the \emph{single} confidence set $\mathcal{P}_t^{j}(d_t^{P_{j}})$ for each episode.
Maximizing these separately yields a (more)
optimistic estimate of each $V_\ell$, making it possibly larger than the actual $V^*$. 
Indeed, the argmax over our relaxed $R$ and $P$ can only make the RHS of \eqref{eq:oracle_obj} larger, which in turn makes the RHS of the inequalities in \textbf{Fig}.~\ref{fig:lp_factored} larger.

% \color{blue}

\raggedbottom
\begin{remark}
\label{rmk:trans_fn}
Each $P_j(\cdot | s[\text{Pa}(Z_j^h)], a)$ marginal has its own confidence set in $\mathcal{P}_t^{j}(d_t^{P_{j}})$, and only depends on the inner sum over $\text{Val}(Z_j^h)$ within each constraint in \textbf{Fig}.~\ref{fig:lp_factored}. This is essential.% to remove imprecision.
\end{remark}

% \raggedbottom

\subsection{Algorithm for Separation Oracle}
We now describe the algorithm implementing the separation oracle for solving \textbf{Fig.~}\ref{fig:lp_factored}. 
We repeat the following  for each action $a \in \mathcal{A}$ separately.
% , so we need only consider the constraints over each state.

\subsubsection{Computing Optimistic Parameters}
If all constraints in \textbf{Fig}.~\ref{fig:lp_factored} are satisfied, the tightest constraint in particular is satisfied. If a constraint is \textit{not} satisfied, then this constraint can be returned for $\mathbf{w}$. 
Our algorithm checks whether the following inequalities hold for each action $a$, obtained by rewriting the constraints in  \textbf{Fig}.~\ref{fig:lp_factored}:%\useshortskip
% \begin{align}
%  %   \begin{split}
%     \label{eq:oracle_obj}
%     &0 \geq %\myquad[2] 
%     \underset{\mathclap{\substack{
%     s \in \mathcal{S}, \overline{R}_i \in \mathcal{R}_t^i\\ 
%     P_j(\cdot | s[\text{Pa}(Z_j^h)], a) \in \mathcal{P}_t^j
%     }
%     }}{\max}\myquad[2]
%     \left[ 
%         \sum_{i=1}^l \overline{R}_i(s,a)
%         + \sum_{j=0}^\phi 
%         \left( - w^{(i)}_j h_j(s)  %\\ 
%             %&
%             %\myquad[3] 
%             + w^{(i+1)}_j \sum_{\hat{s}' \in \text{Val}(Z_j^h)} h_j(\hat{s}') 
%             P_j(\hat{s}'| s[\text{Pa}(Z_j^h)], a)
%         \right)  
%     \right]
% %    \end{split}
% \end{align}
\begin{equation}
    \label{eq:oracle_obj}\small
    \begin{split}
    0 \geq \myquad[1]
    \underset{\mathclap{\substack{
    \ell \in [\tau], 
    s \in \mathcal{S}, \overline{R}_i \in \mathcal{R}_t^i
    P_j(\cdot | s[\text{Pa}(Z_j^h)], a) \in \mathcal{P}_t^j
    }
    }}{\max}\myquad[6]\ \ 
    \left[ 
        \sum_{i=1}^l \overline{R}_i(s,a)
        + \sum_{j=0}^\phi 
        \left( - w^{(\ell)}_j h_j(s) \right. \right.
            \left. \left.
            + w^{(\ell+1)}_j \sum_{\mathclap{\hat{s}' \in \text{Val}(Z_j^h)}} h_j(\hat{s}') 
            P^{(\ell+1)}_j(\hat{s}'| s[\text{Pa}(Z_j^h)], a)
        \right)  
    \right]
    \end{split}
\end{equation}
Notice each $\overline{R}_i$ depends only on the subset of state variables given by its scope $Z_i^R$. We can thus precompute the optimal value of $\overline{R}_i(x[Z_i^R])$ for the polynomial number of assignments to $x[Z_i^R]$, represented by $z \in \text{Val}(Z_i^R)$, in $O(1)$ time by using the largest value within the confidence set:%\useshortskip 
\begin{equation}
    \label{eq:reward_fn_opt}\small
    \overline{R}_i(z) = \frac{1}{n_t(z)} \sum_{\tau < t; x_\tau = x} \delta{y_\tau} + \sqrt{\frac{d_t}{n_t(z)}},
\end{equation}
where -- by abuse of notation -- $n_t(z)$ denotes the number of visits to any $(s,a)$ which takes the values given by $z$ over state variables in $Z_i^R$, up until time $t-1$. Notice that this allows us to fix optimistic values for the rewards in $O(lm)$ time by creating a polynomial-sized lookup table for the value of $\overline{R}_i$ at any $(s,a)$ constraint.

We would like to use a similar procedure to determine an optimistic transition function. For each $j$ in \eqref{eq:oracle_obj}, given $a$, there are multiple transition marginals to solve for, where each depend only on an assignment $z \in \text{Val(Pa(}Z_j^h))$ to the parents of the $j$th scope (\textbf{Rmk}.~\ref{rmk:trans_fn}). Therefore, we have the following optimization problem over each $P^{(\ell)}_j(\cdot | s[\text{Pa}(Z_j^h)], a)$:%\useshortskip
\begin{equation*}
\label{eq:trans_fn_opt}\small
\begin{aligned}
\max_{P}
%  - w^{(i)}_j h_j(s) + 
w^{(\ell)}_j \sum_{\mathclap{\hat{s}' \in \text{Val}(Z_j^h)}} 
    h_j(\hat{s}') P^{(\ell)}_j(\hat{s}'| s[\text{Pa}(Z_j^h)], a) 
\end{aligned}
\end{equation*}
subject to the constraint that $ P^{(\ell)}_j(\cdot | s[\text{Pa}(Z_j^h)], a) \in \mathcal{P}_t^j$. 
As $\mathcal{P}_t^j$ is a convex set (for a given marginal), we can use a variation of Figure~2 of \cite{jaksch} to solve this problem. 
To maximize a linear function over a convex polytope, we need only consider the polynomial number of polytope vertices. Our \textbf{Alg}.~\ref{algo:trans_fn_opt} given in \textbf{Appx.~}\ref{appx:variable_elimination} simply greedily assigns resources to high valued $h_j(s'_k)$ functions, while normalizing to ensure that $P$ remains a true probability distribution.

% The runtime of \textbf{Algorithm}~\ref{algo:trans_fn_opt} is bounded by the sorting of $|\text{Val}(Z_j^h)|$ elements, which, assuming a $\log m$ size $\text{Pa}(Z_j^h)$, is $O(m\log m)$. 

\begin{remark}
We only compute the optimistic parameters for both the reward and transition functions a single time before solving \textbf{Fig}.~\ref{fig:lp_factored}. 
Notice the optimistic reward did not depend on 
% a particular setting of 
$\mathbf{w}$, so we can use the resulting values for each later call to the separation oracle algorithm. 
Similarly, optimistic transition probabilities depend only on $\operatorname{sign}(w^{(\ell)}_j)$ in \textbf{Alg}.~\ref{algo:trans_fn_opt}, 
which means there we only need to compute \emph{at most two} $P^{(\ell)}_j$ for each $j$. 
For each of $N$ transition marginals, we compute and store both orderings based on $\operatorname{sign}(w_j^{(\ell)})$ in the lookup table. 
To check \eqref{eq:oracle_obj} for a query $\mathbf{w}$ in the algorithm, we use transition functions corresponding to the correct ordering in $O(1)$ by table lookup.
\end{remark}

% Importantly, optimizing over transition function marginals independently does not necessarily ensure that the full transition distribution will be consistent with each joint and independent marginal. However, our value function will still be \textit{optimistic} in the sense that it is an over-estimator of the optimal value function $V^*$, 
% {\color{blue}
% which is all that is required in our analysis (please see Sec.~\ref{sec:regret_analysis}).)
% }
% % (\textbf{Rmk}.~\ref{rmk:non_consistency}).

\subsubsection{Variable Elimination}
We now have a polynomial-size lookup table for each possible $\overline{R}_i(x[Z_i^R])$ and $P^{(\ell)}_j(\cdot | s[\text{Pa}(Z_j^h)], a)$. However, we are still left with a maximization over an exponential sized state space $\mathcal{S}$ in \eqref{eq:oracle_obj}. To ameliorate this, we utilize the procedure of \textit{variable elimination} from probabilistic inference, which was applied to FMDPs by \cite{guestrin}.

Variable elimination constructs a new optimization problem $\Omega$, equivalent to \eqref{eq:oracle_obj}, but over a tractable constraint space. Let some order over $\mathcal{S}_1, \dots, \mathcal{S}_m$ be given, and assume that our state space is $\{0, 1\}^m$. 

% {\color{blue}
% Let $c_j$ be the term beside $w_j$ (inside the parenthesis) in \eqref{eq:oracle_obj}.
Based on \eqref{eq:oracle_obj}, we define $c^{(\ell)}_j(s, a)$ as:
\begin{equation*}\small
% c_j(s, a)= 
- w^{(\ell)}_j h_j(s)  %\\ 
            %&
            %\myquad[3] 
            + w^{(\ell+1)}_j \sum_{\mathclap{\hat{s}' \in \text{Val}(Z_j^h)}} h_j(\hat{s}') 
            P^{(\ell+1)}_j(\hat{s}'| s[\text{Pa}(Z_j^h)], a). 
\end{equation*}
Without loss of generality, we will only use one $c_j(s,a)$ to demonstrate the variable elimination, because the variable elimination order is only controlled by the scopes $Z^h_j$ indexed by $j$, so procedure is the same for each $c^{(\ell)}_j(s,a)$. 
% Motivated by the example in \cite{delgado}, 
We illustrate one step of the variable elimination, and the rest follow similarly. Suppose that the only scopes containing $\mathcal{S}_1$ are $Z_1^R = \{ \mathcal{S}_1 \}$, and $\text{Pa}(Z_1^h) = \{ \mathcal{S}_1, \mathcal{S}_4 \}$. Suppose that the first state variable to eliminate is $\mathcal{S}_1$. Variable elimination rewrites \eqref{eq:oracle_obj} by moving the ``relevant functions" inside (due to linearity):%\useshortskip
\begin{align}\small\begin{split}
    \underset{s \in \bigotimes_{i=2}^m \mathcal{S}_i}{\max} \bigg[ &\sum_{i=2}^l \overline{R}_i(x[Z_i^R]) + \sum_{\mathclap{j=0, 2\dots \phi}}  c_j(s, a)\\ 
    % \nonumber
    % \myquad[3]  
    &+ \underset{\mathcal{S}_1}{\max}\Big[ \overline{R}_1(x[Z_1^R]) + c_1(s,a) \Big] \bigg]
    \end{split}
\end{align}
Next, we replace 
% the maximum inside the bracket 
$\underset{\mathcal{S}_1}{\max}[ \overline{R}_1(x[Z_1^R]) + c_1(s,a)]$
with a new LP variable $u_{\mathcal{S}_1}^{e_r}$. However, to enforce $u_{\mathcal{S}_1}^{e_r}$ to be the max, we need to add four additional linear constraints 
% to the LP objective, 
in the form of 
$u_{\mathcal{S}_1}^{e_r} \geq \overline{R}_1(x[Z_1^R]) + c_1(s,a)$, 
one for each binary assignment to $\mathcal{S}_1, \mathcal{S}_4$. 
These constraints involve evaluating $\overline{R}_1(x[Z_1^R])$ and $c_1(s,a)$ at each assignment, which simply uses our previously-constructed $\poly$ sized lookup table. 
(details in \textbf{Appx.}~\ref{appx:variable_elimination}).
% {\color{blue}
% (taking into account $\operatorname{sign}(w_1)$).
% }

In the general case %with $L$ relevant functions, 
the complexity of such variable elimination has an exponential dependence on the \textit{width} of the \textit{induced cost-network} of our scopes. 
Let the set of all scopes $Z = \{Z_i^R \mid i\in[l]\} \cup \{\text{Pa}(Z_j^h) \mid j\in[\phi]\}$ be given. We can construct a cost network over variables $\mathcal{S}_1, \dots, \mathcal{S}_m$ s.t.\ there is an undirected edge between any two variables iff they appear together in any scope in $Z$. The width of this network is the longest path between any two variables.

% \begin{remark}
% In general, we may have network width $m$ and no complexity reduction from exponential time. Environment scopes must be small
% % $O(\log m)$ 
% \textit{and} also ensure the induced cost network has small diameter \footnote{Cost network and FMDP diameters are separate and unrelated.}.
% Heuristics\textsuperscript{\ref{footnote:variable_elim}} could be useful to find an order criterion, but an optimal choice is NP-hard in general.
% \end{remark}

% As variable elimination is well studied in the literature, we defer a full discussion to \textbf{Appx}.~\ref{appx:variable_elimination}, where we offer a worked example and full algorithm.

% Can probably remove this from workshop
\begin{theorem}
\label{them:oracle_poly}
Given an efficient variable elimination ordering over the induced cost network, a polynomial-time (strong) separation oracle exists.
\end{theorem}
\textit{Proof}: 
For a given $\mathbf{w}$, obtain the simplified version $\Omega$ of the exponentially large LP formulation through variable elimination as above (\textbf{Alg.~}\ref{algo:simplify_oracle_obj}). 
Given $\mathbf{w}$, we can efficiently check the feasibility of the original LP by checking feasibility of $\Omega$. 
If $\Omega$ is infeasible for $\mathbf{w}$, then we obtain a sequence of tight simplified linear constraints with the final exceeding the bound of \eqref{eq:oracle_obj}.
Since simplified constraints are obtained by iteratively maximizing state variables, from these tight constraints we can read off the corresponding state variable values $s^*$. 
The inequality in \eqref{eq:oracle_obj} with $s^*$ is the one that $w$ violates, and we use this to define a separating hyperplane, 
which follows from tightness of the new optimization problem and Thm.~4.4 of \cite{guestrin}.

This implies planning in \textbf{Alg.}~\ref{algo:opt_planner} is efficient (\textbf{Appx}.~\ref{proof:oracle_poly}).

% \section{Regret Bound}
% By using analysis similar to \cite{osband}, we can also derive the following regret bound for \textbf{Alg.}~\ref{algo:ucrl_factored}, UCRL-Factored for FSMDPs.
% We modify their analysis by replacing their use of FMDP product transition structure with a factored linear basis assumption on the value function.
% However note that we \emph{do not} need to assume that each $V^{M^*}_{\mu^M, i}$, which is value function when we apply $\mu^M$ instead of $\mu^*$ to $M*$, is linear (see \textbf{Eq.}~\eqref{eq:V_k_k_i_V_star_k_i}, \eqref{eq:V_k_k_i_V_star_k_i_next}, \eqref{eq:T_k_i_T_star_k_i} in \textbf{Appx.}~\ref{appx:regret_analysis}).
% \begin{theorem}
% \label{corr:regret_bound}
% Let $M^*$ be a FSMDP with $V_*^*$ having a linear decomposition. Let $l + 1 \leq \phi$, $C = \sigma = 1$, $|\mathcal{S}_i| = |\mathcal{X}_i| = \kappa$, $|Z_i^R| = |\text{Pa}(Z_i^h)| = \zeta$ for all $i$, and let $J = \kappa^\zeta$, $\| \mathbf{w} \|_1 \leq W$, and $\max_{j\text{ and }s \in \text{Val}(Z_j^h)} |h_j (s)| \leq G$. Assuming $WG \geq 1$, and an efficient variable elimination ordering, then $\text{Regret}(T, \pi_\tau, M^*)$
% $%\begin{align*}
%     \leq \tau \bigg(
%         30\phi WG \sqrt{T J (J\log(2) + \log(2N\zeta T^2 / \delta))}
%     \bigg)\label{eq:finalbound2}
% $%\end{align*}
% w.p. at least $1 - 3\delta$. 
% % where $K = \lceil T / \tau \rceil$ is the number of observed episodes.
% \end{theorem}
% Please see \textbf{Appx.}~\ref{appx:regret_analysis} for proof details. 

\subsection{Completing the Cutting Plane Analysis}

% {\color{blue}
% Maybe we only keep the part about objective summation and leave the technical part about Ellipsoid to appendix.
% }

%Combining the variable elimination procedure with the aforementioned imprecision removal gives us 
By standard arguments, any separating hyperplane may be made \textit{strict} by a perturbation.
We thus obtain a \textit{strong} separation oracle, which returns $\mathbf{w}$ if it lies in the solution set, or a strict separating hyperplane
% \footnote{The \textit{strictness} may be achieved by $\epsilon$-perturbation of a non-strict separator.} 
whose half-space contains the feasible solution set and does not contain the query point $\mathbf{w}$.

% We establish the remaining properties of the Ellipsoid Method for %optimization 
% \textbf{Fig}.~\ref{fig:lp_factored}.

We now establish the   objective  can be evaluated efficiently for %optimization 
\textbf{Fig}.~\ref{fig:lp_factored}.
First, recall that $\underset{\mathbf{w}}{\min} \sum_{s \in \mathcal{S}} \sum_{j=0}^\phi w_j^{(1)} h_j(s)$ is the objective of our problem. A na\"{\i}ve summation over states may require exponential time, so we simplify:
\begin{equation*}\small
    \sum_{j=0}^\phi w_j^{(1)} \sum_{s \in \mathcal{S}} h_j(s)
    = \sum_{j=0}^\phi w_j^{(1)} g(Z_j^h) \sum_{s_k \in \text{Val}(Z_j^h)} h_j(s_k) 
\end{equation*}
where $g: \{Z_j^h\ |\ j\in[\phi] \} \mapsto \mathbb{Z}^+$ counts the number of states that take value $h_j(s_k)$ by counting combinations of state variables which are not in the scope of $h_j$:
$
%\begin{equation*}
    g(Z_j^h) = \prod_{i=1, \dots, m \notin Z_j^h} |\text{Val}(\mathcal{S}_i)|.
$
%\end{equation*}
We can now evaluate our objective in polynomial time by iterating only over states within the scope of each $h_j$. 

Next, the Ellipsoid algorithm also requires that $\mathbf{w}$ lies in a bounded convex set. We will assume $\| \mathbf{w} \|_1 \leq W$ for some $W \in \mathbb{R}$, for reasons discussed further in \textbf{Sec}.~\ref{sec:regret_analysis}. It is clear that if the MDP is well defined and has a bounded linear value function, then $\mathbf{w}$ must be bounded. Our main planning result follows.
\begin{theorem}
\label{them:planner_result}
The Ellipsoid algorithm solves
the optimization problem \textbf{Fig}.~\ref{fig:lp_factored}
in polynomial time.
\end{theorem}
This follows from the strong separation oracle of \textbf{Thm.~}\ref{them:oracle_poly}, but we defer the details to \textbf{Appx}.~\ref{appx:ellipsoid_convergence}.

\subsection{Runtime}
For each episode, the optimistic $P, R$ for all scopes are precomputed in time $O(\tau |A| J \phi)$  (please see \textbf{Thm.~}\ref{corr:regret_bound} for notations). The state-of-the-art convex program solver of \citet{faster_cutting_plane} takes a separation oracle for a convex set $K \subset \mathbb{R}^n$, where $K$ is contained in a box of radius $R$, and finds the optimum in $K$ up to error $\epsilon$ in $O(n \log(nR / \epsilon))$ oracle calls, taking an additional $O(n^2)$ steps per call.
In our case, $n = \tau\phi$, because we are searching for $\phi$-dimensional linear weights $\mathbf{w} \in \mathbb{R}^\phi$ for each step in the episode, and $R \leq O(\tau \phi W)$ because $\|\mathbf{w}\|_1 \leq W$.
% (please see Thm.3 for notation letters).
% In the proof of Thm.~1 in Appx.~B.4, we discuss the separation oracle.
The runtime of the separation oracle is $|A|$ times the cost of solving the small LP after variable elimination.
% \citet{Cohen2021SolvingLP} can solve LPs with $n$ variables to relative accuracy $\delta$ in time $O^*(n^\omega \log(n / \delta))$, where  $\omega$ is the matrix multiplication time exponent.
\citet{Cohen2021SolvingLP} can solve LPs with $n$ variables to relative accuracy $\delta$ near time $O(n^{2.5} \log(n / \delta))$ using fast matrix multiplication algorithms.
% This is near $O(n^{2.5} \log(n / \delta))$ using fast matrix multiplication algorithms.

The variable elimination procedure introduces $n \leq O(\tau m \kappa^{\omega})$ variables into our reduced LP, where $m$ is the number of state variables (not states), and $\omega$ 
is the small induced the width of the \textit{cost network} of the scopes,  
% (discussed before the statement of Thm.~1 in the main paper, where we assume \omega is small)
so our separation oracle runs in time $O(|A| (\tau m \kappa^{\omega} )^{2.5} \log(\tau m \kappa^{\omega} / \delta))$. 
% $= O(|A| \omega (\tau \kappa^{\omega})^{2.5} \log(\tau) )$.
Therefore, the planner runtime for each episode is 
% $O(n^2) \cdot O(n \log (\kappa)) = $
$O(\tau |A| J \phi + |A| (\tau m \kappa^{\omega})^{2.5} \log(\tau m \kappa^{\omega} / \delta) \tau \phi \log( \tau^2\phi^2 W /\epsilon) + \tau^3\phi^3 \log (\tau^2\phi^2 W /\epsilon) )$.

\section{Regret Analysis}
\label{sec:regret_analysis}

By using an analysis similar to \cite{osband}, we can also derive the following regret bound for \textbf{Alg.}~\ref{algo:ucrl_factored}, UCRL-Factored for FSMDPs (details in \textbf{Appx.}~\ref{appx:regret_analysis}). 
\begin{theorem}
\label{corr:regret_bound}
Let $M^*$ be a FSMDP with $V_*^*$ having a linear decomposition. Let $l + 1 \leq \phi$, $C = \sigma = 1$, $|\mathcal{S}_i| = |\mathcal{X}_i| = \kappa$, $|Z_i^R| = |\text{Pa}(Z_i^h)| = \zeta$ for all $i$, and let $J = \kappa^\zeta$, $\| \mathbf{w} \|_1 \leq W$, and $\max_{j\text{ and }s \in \text{Val}(Z_j^h)} |h_j (s)| \leq G$. Assuming $WG \geq 1$, and an efficient variable elimination ordering, then 
\begin{align*}
    \text{Regret}(T, \pi_\tau, M^*)
    \leq \tau \bigg(
        30\phi WG \sqrt{T J (J\log(2) + \log(2N\zeta T^2 / \delta))}
    \bigg)\label{eq:finalbound2}
\end{align*}
w.p.\ at least $1 - 3\delta$. 
% where $K = \lceil T / \tau \rceil$ is the number of observed episodes.
\end{theorem}

\begin{remark}
There is a lower bound example in \cite{xu2020FMDP} that shows such a dependence on $J$ is necessary. Their example extends 
Jaksch et al.\ (2010)
% \citet{jaksch} 
where there are two states $s, s'$, and $r(s, a ) = 0, r(s', a)=1$ for any action $a$. This can be changed to only receiving penalty at $s$, thus giving linear $V^*\equiv0$, which is captured by linear $V$ FSMDP.
\end{remark}

We modify analysis of \cite{osband} by replacing their use of the FMDP product transition structure with a factored linear basis assumption on the value function.
Let $V^{M^*}_{\mu^M, \ell}$ represent the resulting value function after applying policy $\mu^M$ instead of $\mu^*$ to $M^*$. Importantly, we note that we \textit{do not} need to assume each $V^{M^*}_{\mu^M, \ell}$ is linear (\textbf{Eq.}~\eqref{eq:V_k_k_i_V_star_k_i}-\eqref{eq:T_k_i_T_star_k_i} in \textbf{Appx.}~\ref{appx:regret_analysis}).

% For simplicity of illustration, we first give an analysis for the discounted episodic setting, then obtain the regret for the non-discounted episodic setting by revising it. 
% We sketch the regret analysis of \textbf{Alg}.~\ref{algo:ucrl_factored}, which is similar to the analysis of \cite{osband}, and defer details to \textbf{Appx}.~\ref{appx:regret_analysis}.
To start our analysis,  we denote: %\citep{osband_dp}
\begin{equation*}
\small
    \mathcal{T}^M_{\mu^M,\ell} V(s) = 
    \overline{R}^M(s, \mu(s))
     + \sum_{s' \in \mathcal{S}} P^M(s' | s, \mu(s)) V(s'),
\end{equation*}
We simplify our notation by writing $*$ in place of $M^*$ or $\mu^*$ and $k$ in place of $\tilde{M}_k$ and $\tilde{\mu}_k$.
Without loss of generality, we examine the regret of an episode starting from each given state. Let $s_{t_k + 1}$ be the first state in the $k$th episode.  
The regret of the $k$th episode is then given by 
$\Delta_k = V_{*,1}^{*}(s_{t_k + 1}) - V_{k,1}^{*}(s_{t_k + 1})$. 
Note that $P^*$ is homogeneous throughout the episode, but there are distinct (optimistic) estimates $\{ P^{k, (\ell)} \}_{\ell \in[\tau]}$ for each step. 
We will also denote  $x_{k, i} = (s_{t_k + i}, \mu_k(s_{t_k + i})) $. 
Importantly, $V_{k,\ell}^* = \mathcal{T}_{k, \ell}^* V_{k,\ell+1}^*$ because here we are applying the action of $\mu^k$ to the actual environment of $M^*$, and $V_{k,\ell}^k = \mathcal{T}_{k, \ell}^k V_{k,\ell+1}^k$ because at the optimal solution, the LP constraints are tight. 
% in \eqref{eq:regret_episode_k}. 
% We split the regret into five bounded parts: (1) planning accuracy, (2) a martingale, (3) reward function estimate, (4) transition function estimate, and (5) PAC concentration bound. 
% We combine these to derive a regret bound in terms of the number of \textit{episodes}. 

First, let's add and subtract the computed optimal reward:%\useshortskip
\begin{align*}\small
% \Delta_k =
&  V_{*,1}^{*}(s_{t_k + 1}) - V_{k,1}^{*}(s_{t_k + 1})=
\\
&( 
    V_{k, 1}^k(s_{t_k + 1})\! -\! V_{k, 1}^*(s_{t_k + 1}))\!+\!(
    V_{*, 1}^*(s_{t_k + 1})\! -\! V_{k, 1}^k(s_{t_k + 1})), 
\end{align*}
where the second term on the RHS can be bounded by a choice of planning error $\epsilon = \sqrt{1/k}$.  
Indeed $V^k_{k,1}$ without planning error can only overestimate $V_{*,1}^*$ by optimism. 
% , implying that this term is at worst
% $\leq 0 + \epsilon = \sqrt{1/k}$. 

% of $V^k_{k,1}$
% and correctness of our provided planner 
% (\textbf{Thm}.~\ref{them:planner_result}).
% Indeed, the argmax of relaxed $R, P$ can only make the RHS of \eqref{eq:oracle_obj}, which in turn makes the RHS of the inequalities in \textbf{Fig}.~\ref{fig:lp_factored} larger.
% Therefore,
% $V^k_{k,1}$ can only overestimates $V^*_{*,1}$, implying that this term is at worst
% $\leq 0 + \epsilon = \sqrt{1/k}$. 
% \begin{remark}
% \label{rmk:non_consistency}
% Importantly, the second term can only become \textit{more} negative when we do not insist that the transition marginals are consistent (in that they represent the marginals of a real distribution).
% This allowed us to relax enforcement of the consistency of the marginals within our proposed oracle.
% % As the procedure refines the confidence sets over episodes, the empirical estimate of each transition marginal will better capture the true marginals, which 
% % are necessarily consistent with one another.
% \end{remark}

Now let's
deconstruct the first term on the RHS  above through dynamic programming \citep{ osband}:%\useshortskip
\begin{equation}
    \label{eq:main_applying_dp}\small
    % = \frac{1}{1 - \gamma}\bigg( 
    % (\mathcal{T}_k^k - \mathcal{T}_k^*) V_k^k (s_{t_k + 1}) + d_{t_k}
    % \bigg),
   = \sum_{\ell=1}^\tau (\mathcal{T}_{k, \ell}^k - \mathcal{T}_{k, \ell}^*)V_{k, \ell+1}^k (s_{t_k + \ell}) + \sum_{\ell=1}^{\tau} d_{t_k + \ell}, 
\end{equation}
where $d_{t_k+\ell}$ is a martingale difference bounded by $\max_{s \in \mathcal{S}} V^k_{k,\ell+1}(s) $, which in turn is bounded by 
$B_{w,h} = 
\| \mathbf{w} \|_1  \max_{s \in \mathcal{S}} \max_j | h_j(s) |$
% $\| \mathbf{w} \|_1$ and the max value of any $h_j$ over all states $s \in \mathcal{S}$, 
due to H\"older's inequality being applied to the linear form of the computed $V^k_{k+i}$. 
Next, similarly to \cite{jaksch}, we apply Azuma–Hoeffding to obtain $\sum_{k=1}^{\lceil T/ \tau\rceil}\sum_{i=1}^{\tau} d_{t_k + i} \leq O(B_{w,h} \sqrt{T})$ w.p.\ $\geq 1-\delta$.
(However, we now have $B_{w,h}$ instead of a dependence on the MDP diameter.)  
% due to our norm bounded linear basis approach. 
% due to the H\"older's inequality. 

For the remaining terms in \eqref{eq:main_applying_dp}, we apply Cauchy-Schwarz to obtain the following bound:%\useshortskip
\begin{equation}\small
    % \leq
    % &|\overline{R}^k(x_{k, 1}) - \overline{R}^*(x_{k, 1})| 
    % +  \sum_{j=1}^\phi \bigg| w_k^{k, (j)} %\nonumber\\
    % %&\myquad[2] 
    % \sum_{s' \in \mathcal{S}} (P^k (s' | x_{k, 1}) \nonumber \\
    % & -  P^* (s' | x_{k, 1})) h^{(j)}(s') \bigg|
    % \label{eq:main_analysis_target}
    \begin{split}
    & \leq 
    \sum^{\tau}_{\ell=1}
    \left[ 
    \left|\overline{R}^k(x_{k, \ell}) - \overline{R}^*(x_{k, \ell})\right|+
    \right.
     \label{eq:main_analysis_target} \\
    & \left.
      \sum_{j=1}^\phi \bigg| w_{k, j}^{k, (\ell+1)} \!\sum_{s' \in \mathcal{S}} (P^{k,(\ell)} (s' | x_{k, \ell}) -  P^* (s' | x_{k, \ell})) h_{j}(s') \bigg|\ \!\right]
    \end{split}
\end{equation}
The difference between the actual reward $\bar{R}^*$ and computed reward $\bar{R}^k$ in \eqref{eq:main_analysis_target} are bounded  by the widths of reward confidence sets, akin to \cite{osband}. 
% where we use the concentration guarantee for the confidence sets of reward. 
For the rest of terms in \eqref{eq:main_analysis_target}, 
we diverge from \cite{osband} by applying a different H\"older's inequality argument
%with different choice of norms.
% w.r.t. the vectors $w, h, $ and $P^k-P^*$. 
which results in a bound with respect to
% the $L_1$ norm on $\mathbf{w}$, $L_\infty$ norm on $h_j$,
$ \| \mathbf{w} \|_1 \max_{s \in \mathcal{S}} \max_j | h_j(s) |$
and $\|P^k - P^*\|_1$, 
 where the latter is bounded by the widths of transition probability confidence sets similar to those of the rewards.
% Since our value function is linear over a finite number of states, it must be the case that it is bounded and therefore each $h_j$ is also bounded.
% We have now bounded our estimated reward and transition functions w.r.t. their true counterparts in terms of the $L_1$ norm.
We then apply \textbf{Corollary~}\ref{corr:bounded_widths} from \textbf{Appx.}~\ref{appx:regret_analysis} to bound the widths of confidence sets over time, which uses the concentration bound 
 $\leq O(\textit{poly}(J) \sqrt{T} )$ 
for each confidence set. 
% which polynomially depends on scope size bound $\zeta$ plus a concentration error $\leq O(\sqrt{T})$. 
\iffalse{
Lastly, using $L_1$ empirical deviation bounds, we show that for a choice of $\delta$ at each episode,
the true MDP lies within our confidence sets w.p. at least $1 - 2 \delta$.
}\fi
% \begin{equation}
%     \mathbb{P}\bigg(M^* \in \mathcal{M}_k\ \forall k \in \mathbb{N} \bigg) \geq 1 - 2\delta.
% \end{equation}
% This completes all the tools needed to derive a PAC regret bound.
\iffalse{ 
We defer the general bound to 
\textbf{Thm}.~\ref{appx_thm:main_regret_bound} in \textbf{Appx}.~\ref{appx:full_regret_thm}.
}\fi

% We then revise our analysis for the non-discounted episodic setting:
% after solving the LP with the new additional constraints, at each $i$th step $V_i = \mathcal{T}_{i,a} V_{i+1}$ will be tight for some action $a$. 
% % On the other hand, for actual optimal values we have $V^*_i = \mathcal{T}^*_{i,a^*} V^*_{i+1}$
% By repeatedly applying this time-dependent Bellman operator
% % starting from $V^*_1 - V_1$ and plugging it 
% in our regret analysis, we verify that the only difference from the discounted setting is that we obtain a summation over $i=1,\dots, \tau$, which eventually introduces $\tau$.
% Due to the complex formulas within the main theorem, we present a simplified version for a symmetric case, 
% which we present in \textbf{Thm}.~\ref{corr:regret_bound} (details in \textbf{Appx}.~\ref{appx:extension_analysis}).

\paragraph{Discussion}
Our bound is similar to Cor.~2 from \cite{osband}, but not identical. Most importantly, we have \textit{provided} an efficient planning algorithm which \cite{osband} \textit{assume} as an oracle when computing their regret bound.
% {\color{blue}
% Second, since we are considering a modified discounted scenario, we have a natural $\frac{1}{1-\gamma}$ term in our bound (although our method obtains an analogous result in the finite episode setting \textbf{Sec}.~\ref{sec:extend_results}).
% }
% Next, we have a $\sqrt{K}$, the number of true episodes, rather than the bandit lower bound $\sqrt{T}$.
We also have an extra $\sqrt{J}$ cost due to the support of each the transition marginal functions we are estimating being of size $J$ and not of size $\kappa$. 
This follows naturally from considering transition functions that don't fully factorize. 

Instead of a dependence on the number of state variables $m$, we have a factor of $\phi$, the number of basis functions. 
\cite{osband} have a factor of the \textit{diameter} in their formal guarantee.
% simplified bound, which they replace with $\tau$ the time horizon. 
However, our bound \textit{does not} rely on the diameter and instead depends only on the 1-norm of the basis vector $W$ and the max value that any basis function $G$. 
% In fact, neither our planner nor regret analysis require the FMDP be completely connected as long as there is a reachable state in every connected component (see \textbf{Rmk}.~\ref{appx_rmk:diameter_independent} in \textbf{Appx}.~\ref{appx:full_regret_thm} for details).
% % This matches the information-theoretic lower bound of \cite{sra2020FMDP}.
Our dependence on the horizon $\tau$ rather than diameter matches the recent minimax-regret of \cite{sra2020FMDP} for finite episode RL in FMDPs. However, our regret bound can be obtained using a provably efficient algorithm (without assuming there is an oracle that efficiently iterates through every state). 

% Lastly, our only dependence on $\tau$ is to determine the number of episodes, $K = \lceil T / \tau \rceil$.

\section{Future Work}

% It remains open whether the $\Omega(\sqrt{SAT})$ lower bound from \cite{jaksch} extends to our linear value correlated FMDP setting.
No lower bound for our problem setting is known.  
Our regret bound in \textbf{Thm}.~\ref{corr:regret_bound} is polynomial in $\phi$, which represents the size of our basis. 
We also ask if it is possible to remove this dependency on $\phi$ in our transition function error analysis \eqref{eq:main_analysis_target}.
%With additional theoretical work in our transition function error analysis \eqref{eq:main_analysis_target}, we may able to remove this dependency on $\phi$. 
This would allow for our approach to be utilized with kernels and a possibly infinitely sized basis. 
Correspondingly, we ask if there exists an efficient kernelized planning algorithm; if both could be resolved affirmatively, this would in turn enable the use of rich, kernelized value functions (as opposed to Q-functions) for RL in large FMDPs.

\subsubsection*{Acknowledgements}
This research is partially supported by NSF awards IIS-1908287, IIS-1939677, and CCF-1718380, and associated REU funding.
We thank the anonymous AISTATS reviewers for their helpful comments and discussions.

%%%%%%%%%%%%%%%%%%%%%%%%%%%%%%%%%%%%%%%%%%%%%%%%%%%%%%%%%%%%

\bibliographystyle{plainnat}
\bibliography{ref}

\appendix

\onecolumn

% \icmltitle{Polynomial Time Reinforcement Learning in Correlated FMDPs with Linear Value Functions: 
% Supplementary Materials}

% \icmltitlerunning{Polynomial Time Reinforcement Learning in Correlated FMDPs with Linear Value Functions: 
% Supplementary Materials}

% \begin{bibunit}
\setcounter{section}{0}
\renewcommand*{\theHsection}{Supplement.\the\value{section}}

\section{Relation Between Linear Value and Linear Q-function}
\label{appx:linear_V_and_linear_Q}
Many recent advances on provable polynomial RL algorithms assumes the state-action value function ($Q$-function) to be linear:
a linear $Q$-function is defined as 
$Q^* = \sum_{i=1}^\phi w_i h_i(s, a)$,
with basis elements $\{h_1, .., h_\phi\}$,
and almost all of them use Least Square Value Iteration (LSVI) based algorithms. 
However, the linear $Q$-function assumption has its limitations.
For example, \citet{yang_linear_q} shows that a linear $Q$ function requires the transition function to be linear in order to avoid unbounded Bellman error. 
% is linear with respect to a basis set of feature functions if and only if the transition function has a linear representation with the same basis 
(A similar argument appears in \textbf{Proposition~2.3, 5.1} of \citet{jordan_linear_rl}).
By contrast, we will show in \textbf{Prop.}~\ref{prop:V_vs_Q_in_main_paper} that a linear value function \emph{does not} entail that the $Q$-function is linear. 
% Therefore, by using linear value function we can capture RL classes with more complicated Q functions that are not necessarily linear.
As a contraposition, it's been shown ( \citep{jordan_linear_rl}, \citep{yang_linear_q} ) that linear transition function implies linear $Q$ function. Therefore in our example, the transition function is not linear either. 

Moreover, value iteration based algorithms with linear $Q$ function require Bellman Error to be zero. This needs to be either explicitly assumed or it requires linear transition function in order for this to be true.
This drastically reduced the practicality of the linear function model.
Since our algorithm is not value iteration based and our problem has a finite-episode, this restriction does not apply to us.

Intuitively, we can have nonlinear $Q$ function while $V$ function being linear because $V(s) = \max_{a} Q(s,a)$ and maximum function being linear does not necessarily infer that piece-wise functions are linear.  
Concretely, for a given state-action basis $\{h_i(s,a)\}_i$, we can provide an MDP for which there is no coefficient setting $\mathbf{w}$ for which the optimal $Q$-function is linear, whereas this MDP will have an optimal linear value function $V^* = \sum_{i=1}^\phi w_i f_i(s)$ for \textit{any} state value function basis $\{f_i(s)\}_i$.

% Consider the family of bases $\mathcal{B}$ with weights in $\mathcal{W} \subset \mathbb{R}$ and functions $\mathcal{H}$ both of at most bit complexity $O(\poly(m))$.
% A compact linear basis $B \in \mathcal{B}$ is then defined as a pair $B = (W, H)$, where $W = \{w_1, .., w_\phi\} \subset \mathcal{W}$, $H = \{h_1, .., h_\phi\} \subset \mathcal{H}$ and furthermore the bit complexity of $w_i$ and $h_i$ are both $O(\poly(m))$.
% We focus on cases where $\phi$ is $O(\poly(m))$, as this admits a \textit{compact} linear representation in the case of an exponential number of states $N$.

% We emphasize in the following \textbf{Prop.~}\ref{prop:V_vs_Q_in_main_paper} that we are interested in comparing the \textit{complexity} of the linear $Q$-function and value function. We \textit{do not} argue that the MDP in the proof does not have a linear $Q$-function representation, we instead argue that it has no \textit{compact}, $O(\poly(m))$ linear $Q$-function representation.
% Recent work with linear $Q$-functions \cite{yang_linear_q, wang2020GeneralValue, jordan_linear_rl, yang2020MatrixBandit} have a polynomial dependence on the dimension $\phi$ of their linear $Q$-function representation. If there is no compact linear $Q$-representation, then their bounds becomes vacuous.

\textbf{Proposition~\ref{prop:V_vs_Q_in_main_paper}.}
Let a state-action ($Q$-function) basis $\{h_1(s,a), .., h_\phi(s,a)\}$ be given such that $\phi < N=2^m$. Then there is an MDP family $\mathcal{M}$ on $N$ states ($m$ binary factors) for which the optimal $Q$-function cannot be expressed as a linear combination of these basis functions with high probability ($1-2^{-N+\phi} \geq 1/2$) for any MDP $M \in \mathcal{M}$. On the other hand, every MDP $M \in \mathcal{M}$ does admit a compact, optimal linear value function representation for \textit{any} given basis set of state feature functions.

\begin{proof}
%Let $m$ be the number of binary state variables, and let there be $N = 2^m$ states.
Consider a family of environments where there are $N$ states $S_1, \dots, S_N$, and for simplicity the time horizon $\tau = 1$.
Pick one of the states and call it $S_{opt}$.
There are two actions everywhere within these MDPs: action $a_1$ takes any state $S_i$ to $S_{opt}$ for all $i \in [N]$ and gives reward $0$; action $a_2$ takes $S_i$ to $S_{j(i)}$,  $j(i) \neq opt$, and gives a reward from the set $\{-1, -1/2\}$.
Call the family of all possible MDPs of this form $\mathcal{M}$. We sample $M\in\mathcal{M}$ uniformly at random---equivalently, by taking $j(i) \sim \mathrm{uniform}(N-1)$ independently for each $i$, and the rewards independently and uniformly from $\{-1, -1/2\}$.

The optimal value function is a constant $0$ for every state in every MDP $M \in \mathcal{M}$.
That is, $V(S_j) = 0$ for all $j \in [N]$ as the optimal policy simply takes the action $a_1$ everywhere -- we can always obtain $0$ by taking $a_1$ and the other action incurs negative reward in all states.
Therefore, the value function can be represented with \textit{any} basis by taking the zero linear combination. 
%This constant basis assumption is standard. For example, \cite{guestrin} require a constant basis function for the planning procedure of Approximate Linear Programming to converge.

%On the other hand, the $Q$-function in general requires $N$ bits to indicate which of $\{-1, -1/2\}$ the MDP $M$ takes for action $a_2$ at each of the $N$ different states.
% On the other hand, for the $Q$-function, since you must consider rewards of the form $\{0, 1/2\}$ for $a_2$, that's one bit per state, so in general you need $N$ bits to indicate which one you chose for all of the $N$ states. 
%Intuitively, this makes the complexity of the optimal $Q$-function is much higher than the value function. 

On the other hand, consider the $\phi \times 2N$ matrix of $Q$-function basis feature representations for each $s,a$ pair:
\begin{equation*}
    B = \begin{bmatrix}
     h_1(S_1, a_1) & h_1(S_2, a_1) & \dots & h_1(S_N, a_1) & h_1(S_1, a_2) & \dots & h_1(S_N, a_2) \\
     h_2(S_1, a_1) & h_2(S_2, a_1) & \dots & h_2(S_N, a_1) & h_2(S_1, a_2) & \dots & h_2(S_N, a_2) \\
     \vdots & \vdots & \ddots & \vdots & \vdots & \ddots & \vdots \\
     h_\phi(S_1, a_1) & h_\phi(S_2, a_1) & \dots & h_\phi(S_N, a_1) & h_\phi(S_1, a_2) & \dots & h_\phi(S_N, a_2)
    \end{bmatrix}
\end{equation*}
Choose a maximal ($d$) size set of states $S' = \{S'_1, \dots, S'_d\}$ such that the column vectors given by 
\newline
$\begin{bmatrix} h_1(S'_i, a_2) & h_2(S'_i, a_2) & \dots & h_\phi(S'_i, a_2) \end{bmatrix}^T$ are linearly independent for all states $S'_i \in S'$ (naturally, $d \leq \phi$ ).
Next, consider any assignment of rewards from $\{-1, -1/2\}$ for these $d$ states, and suppose for contradiction that there is a linear representation of every environment in $\mathcal{M}$.
By assumption, any state $\hat{S} \notin S'$ has $\begin{bmatrix} h_1(\hat{S}, a_2) & h_2(\hat{S}, a_2) & \dots & h_\phi(\hat{S}, a_2) \end{bmatrix}^T$ determined by a linear combination of columns of states in $S'$, given by $\lambda_1,\ldots,\lambda_d$. In particular, supposing that for some choice of $\{w_1,\ldots,w_\phi\}$, $\sum_i w_i h_i(S'_j,a_2)=Q(S'_j,a_2)$ for all $j$, if these also represent $Q(\hat{S},a_2)$, then 
\begin{align*}
Q(\hat{S},a_2)=\sum_{i=1}^\phi w_i h_i(\hat{S},a_2)
=\sum_{i=1}^\phi w_i\sum_{j=1}^d \lambda_j h_i(S'_j,a_2)
=\sum_{j=1}^d\lambda_j\sum_{i=1}^\phi w_i h_i(S'_j,a_2)
=\sum_{j=1}^d\lambda_j Q(S'_j,a_2).
\end{align*}
I.e., $Q(\hat{S}, a_2)$ is therefore determined by rewards of states in $S'$, but we have two distinct, possible values for $Q(\hat{S}, a_2)$ in our family: $\{-1, -1/2\}$. Therefore, MDPs taking one of them cannot be captured by linear functions over the basis.
Furthermore, for an MDP $M \in \mathcal{M}$ chosen at random, since the reward of each $\hat{S}$ is chosen independently, the $Q$-function is linear with probability only $2^{-(N-d)}$. Since $d\leq\phi<N$, this is at most $1/2$.
\end{proof}

We emphasize that we are first given a basis, and are interested in understanding families of environments which may or may not be a linear combination of these bases elements.
We \textit{do not} state that a random MDP from the family we provide does not have its own linear $Q$-function representation. (Indeed, any basis that includes $Q(s,a)$ trivially represents the $Q$ function.) We only state that for a \textit{given} basis, we can find an MDP $M$ whose optimal $Q$-function does not admit a linear decomposition with high probability.

\textbf{Prop.}~\ref{prop:V_vs_Q_in_main_paper} demonstrates that there exist some RL environments where it is feasible to learn a compact linear value function but for which a compact linear $Q$-function is not expressive enough. We remark that conversely to \textbf{Prop.~}\ref{prop:V_vs_Q_in_main_paper}, due to the relationship $V(s) = \max_{a} Q(s,a)$, there surely exist MDPs for which there is a compact linear $Q$-function but no compact linear value function. (It is in general only piecewise linear.) Therefore, we argue that the linear $Q$-function work is orthogonal to ours.

\section{Planner Construction Derivation}
% In this section we first derive the planner w.r.t. to a 
% Then we will switch back to 
% \subsection{Value Function Decomposition Derivation}
% \label{appx:decompisition_derivation}
% By marginalizing similar to \citet{koller} and using \textbf{Def.~}\ref{def:factored_value}, we rewrite the second term of the Bellman operator:\useshortskip
% \begin{align}
%   %&= 
%   \sum_{s' \in \mathcal{S}} P^M(s' | s, a) \sum_{j=1}^\phi w_j h_j (s'[Z_j^h])
%   %\\
%   &= \sum_{j=1}^\phi w_j \sum_{\hat{s}' \in \text{Val}(Z_j^h)}  h_j(\hat{s}') 
%   \sum_{\mathclap{ \bar{s}' \in \text{Val}(\overline{Z}_j^h)}}P^M(\hat{s}', \bar{s}' | s, a) \\
%   &= \sum_{j=1}^\phi w_j \sum_{\hat{s}' \in \text{Val}(Z_j^h)}  h_j(\hat{s}') P^M(\hat{s}'| s, a) \\
%   &= \sum_{j=1}^\phi w_j \sum_{\hat{s}' \in \text{Val}(Z_j^h)}  h_j(\hat{s}') P^M(\hat{s}'| s[\text{Pa}(Z_j^h)], a). \label{eq:backprojection}
% \end{align}
% Here we split $s'$ into $\hat{s}'$ and $\bar{s}'$.
% We first marginalized out $\bar{s}' \in \text{Val}(\overline{Z}_j^h)$, the parts of the state $s$ which are outside of the scope of the $j$th basis function (\textbf{Def}.~\ref{def:factored_value}), then condition $P^M$ on the variables of $s$ that occur in the parents w.r.t.\ $a = \mu(s) \in \mathcal{A}$ of the scope $Z_j^h$.

\subsection{Linear Programming Formulation}
\label{appx:planning_extension}

% It's easy to switch to undiscounted episodic setting because 
We introduce a distinct value function $V_\ell$ for step $\ell$ each episode for the linear programming. 
Concretely, %for the undiscounted episodic setting, 
based on the Bellman operator 
$V_\ell(s) = \max_a \{ R(s, a) + \sum_{s'} P(s' | s, a) V_{\ell+1} (s') \}$, 
we need to solve the following multi-level linear problem with the following constraints 
(for simplicity we do not write out the linear constraints that $R, P$ must be within their respective confidence sets):
$$
\min_{V_{1}} \sum_s V_1(s) 
\ \ \ s.t.  \ \ \ 
V_1(s) \geq R(s, a) + \sum_{s'} P(s' | s, a) V_2 (s'), \ \ 
\forall s \in \mathcal{S}, a \in \mathcal{A},
$$
where $V_2$ is the solution of 
$$
\min_{V_2} \sum_{s} V_2(s)
\ \ \ s.t.  \ \ \ 
V_2(s) \geq R(s, a) + \sum_{s'} P(s' | s, a) V_3 (s'), \ \ 
\forall s \in \mathcal{S}, a \in \mathcal{A},
$$
where $V_3$ is the solution of subsequent subproblem involving $V_4$ with the same structure, and so on. This multi-level linear problem ends with 
$$
\min_{V_\tau} \sum_{s} V_\tau(s)
\ \ \ s.t.  \ \ \ 
V_\tau(s) \geq R(s, a) + \sum_{s'} P(s' | s, a) V_{\tau+1} (s'), \ \ 
\forall s \in \mathcal{S}, a \in \mathcal{A},
$$
where $V_{\tau+1}(s) = 0, \forall s \in \mathcal{S}$ because each episode only has $\tau$ steps. 
These linear programming formulations are equivalent to the step-wise sequential relationship:
$$
V_{\ell}(s) = \max_{a} \left\{  
R(s, a) + \sum_{s'} P(s' | s, a) V_{\ell+1} (s')
\right\},  \ \ \ i = 1, \dots, \tau. 
$$
By inductively following a similar argument as Lemma 1. of \citet{delgado}, we can see that this multi-level linear programming problem is equivalent to the following linear programing problem:
\begin{align}\label{eq:multi_V_LP}
& \min_{V_1} \sum_{s} V_1 (s) 
\\
s.t. \ \ \ &
V_{\ell}(s) \geq R(s, a) + \sum_{s'} P(s' | s, a) V_{\ell+1} (s'), \ \ 
\forall s \in \mathcal{S}, a \in \mathcal{A}, \ \ \ \ell = 1, \dots, \tau, 
\nonumber \\ 
& V_{\tau+1} (s) = 0, \ \ \  \forall s \in \mathcal{S}.  \nonumber
\end{align}

Here each $V_\ell(s)$ has the factored linear form $\sum_j w^{(\ell)}_j h_j(s)$. 
Intuitively, the tightness at the optimal of the LP ``pushes'' $V_\ell(s)$ to be the min of its own corresponding sub-problem.  
% Essentially, we are introducing extra $w_j$ variables to the linear program. 
% Since \eqref{eq:multi_V_LP} has the same linear programming structure as the discounted episodic setting (with single value function $V$), we can still solve it by using our Ellipsoid-based planning algorithm.  

\subsection{Relation to Previous Formulations}
\label{appx:opt_problem_formulation}
Imprecise MDPs are MDPs where the transition function may be defined imprecisely over a bounded convex set. Naturally, this leads to multiple notions of optimality. One such notion is pessimism, where we are interested in the optimal policy in the case where the transition function is always working ``against us" (maximin). \cite{delgado} formulate the maximin solution to imprecise FMDPs by extending \eqref{eq:guestrin_ALP} as follows.
\begin{equation}
\label{eq:delgado_factored_MDP}
\begin{aligned}
\min_{w} \quad & \sum_{\mathbf{x}} \sum_{i=0}^k w_i h_i(\mathbf{x}) \\
\textrm{s.t.} \quad & \sum_{i=0}^k w_i h_i(\mathbf{x}) \geq R(\mathbf{x},a) + \gamma \sum_{\mathbf{x'}\in \mathcal{S}} P(\mathbf{x'}|\mathbf{x},a) \sum_{i=0}^k w_i h_i(\mathbf{x'}), \forall \mathbf{x} \in \mathcal{S}, \forall a \in \mathcal{A} \\
    & P(\mathbf{x'}|\mathbf{x},a) = \underset{Q}{\argmin} \sum_{\mathbf{x'} \in \mathcal{S}} Q(\mathbf{x'}|\mathbf{x},a) \sum_{i=0}^k w_i h_i(\mathbf{x'}) \\
    & \text{where } Q(\mathbf{x'}|\mathbf{x},a) = \prod_{i} Q(x_i'|pa(X_i'), a) \\
    & \textrm{s.t.} \quad Q(x_i'|pa(X_i'), a) \in K_a(X_i' | pa(X_i')) \\
\end{aligned}
\end{equation}
Where $K$ denotes a convex transition credal set.

\cite{guestrin} give a simplification of approximate linear programming (ALP) in the factored case, reducing the number of constraints to allow ALPs to be tractable even with exponentially many states in the MDP. \cite{delgado} applies a similar simplification to the imprecise case, allowing them to heuristically solve imprecise factored MDPs with an exponential number of states. However, due to their product constraint, the problem is non-convex in general and may not find the optimum value function.

Our approach is based upon an insight into the constraints in \eqref{eq:guestrin_ALP}, and utilizes the constraint simplification of \eqref{eq:delgado_factored_MDP} to efficiently and exactly run a linear program to solve for the optimistic solution to the imprecise FMDP defined over our confidence sets.

Let $\mathcal{R}_t^i(d_t^{R_i})$ and $\mathcal{P}_t^{j}(d_t^{P_{j}})$, the reward function and transition function confidence sets at the $t$th time step, be given. Our goal is to generate an $\epsilon$-optimal planner which returns the optimistic solution to the \textit{set} of MDPs given by these confidence sets. Formally, at the $k$th episode of our procedure we would like the optimistic solution to the set of MDPs $M_k$ given as follows.
\begin{equation}
    \label{eq:mdp_set}
    \mathcal{M}_k = \{M | \overline{R}_i \in \mathcal{R}_t^i(d_t^{R_i}), P_j \in \mathcal{P}_t^{j}(d_t^{P_{j}})\ \forall i, \forall j\}
\end{equation}
Where $\overline{R}_i$ is the expected reward of the $i$th $\sigma$-subgaussian factored reward function.

Combining the formulations of \cite{guestrin} and \cite{delgado}, we then obtain the LP formulation for our problem \textbf{Fig}.~\ref{fig:lp_factored}.

\subsection{Constructing a Separation Oracle}
\label{appx:variable_elimination}
Consider the stated separation oracle objective.
% \begin{equation}
%     \label{eq:suppl_oracle_obj}
%     0 \geq 
%     \underset{s \in \mathcal{S}, \overline{R}_i \in \mathcal{R}_t^i, P(\cdot | s[\text{Pa}(Z_j^h)], a) \in \mathcal{P}_t^j}{\max} 
%     \bigg[ 
%         \sum_{i=1}^l \overline{R}_i(s,a) 
%         + \gamma \sum_{j=0}^\phi w_j \big(-h_j(s) 
%             + \sum_{\hat{s}' \in \text{Val}(Z_j^h)} h_j(\hat{s}') 
%             P(\hat{s}'| s[\text{Pa}(Z_j^h)], a)\big) 
%     \bigg] 
%     \forall a \in \mathcal{A}
% \end{equation}
\begin{align}
    \begin{split}
    \label{eq:suppl_oracle_obj}
    &0 \geq \myquad[2] 
    \underset{\mathclap{\substack{
    \ell \in [\tau], 
    s \in \mathcal{S}, \overline{R}_i \in \mathcal{R}_t^i\\ 
    P_j(\cdot | s[\text{Pa}(Z_j^h)], a) \in \mathcal{P}_t^j
    }
    }}{\max}\myquad[3]
    \left[ 
        \sum_{i=1}^l \overline{R}_i(s,a)
        + \sum_{j=0}^\phi 
        \left( - w^{(\ell)}_j h_j(s) \right. \right. 
            %\myquad[3] 
            \left. \left.
            + w^{(\ell+1)}_j \sum_{\hat{s}' \in \text{Val}(Z_j^h)} h_j(\hat{s}') 
            P^{(\ell+1)}_j(\hat{s}'| s[\text{Pa}(Z_j^h)], a)
        \right)  
    \right]
    \end{split}
\end{align}

Notice that maximizing over $s \in \mathcal{S}$ is the same as maximizing over $\mathcal{S}_1, \dots, \mathcal{S}_m$ individually as the state space is factored.
We can then apply the methods from \cite{delgado} and \cite{guestrin} to simplify the maximization procedure. Checking whether \eqref{eq:oracle_obj} is satisfied can be done in two steps, first by solving the exponential sized LP given on the RHS for each $a$, and second comparing the maximum over all $a \in \mathcal{A}$ to 0. 
We will focus on the first step, since the second is trivial. We can group and rewrite the program as follows.
\begin{equation}
    \label{eq:simplified_oracle_obj}
    \underset{A}{\max} \bigg[ \sum_{i=1}^l \overline{R}_i(x[Z_i^R]) + \sum_{j=0}^\phi c^{(\ell)}_j(s, a) \bigg]
\end{equation}
Where $A$ is $\mathcal{S}_1, \dots, \mathcal{S}_m, \overline{R}_i \in \mathcal{R}_t^i, P(\cdot | s[\text{Pa}(Z_j^h)], a) \in \mathcal{P}_t^j\  \forall i = 1 \dots l\ \forall j = 0\dots \phi, \forall \ell = 1,\dots, \tau $, the cartesian product of states, confidence sets for rewards, and confidence sets for marginal distributions. 
Furthermore, $x = (s,a)$ is scoped on the $i$th reward scope, and $c^{(\ell)}_j$ is defined as:
% the appropriate function from \eqref{eq:oracle_obj} inside the parenthesis next to $w_j$, with $\gamma$ brought inside. 
$$
- w^{(\ell)}_j h_j(s)  %\\ 
            %&
            %\myquad[3] 
            + w^{(\ell+1)}_j \sum_{\hat{s}' \in \text{Val}(Z_j^h)} h_j(\hat{s}') 
            P^{(\ell+1)}_j(\hat{s}'| s[\text{Pa}(Z_j^h)], a). 
$$
Without loss of generality, we will only use one $c_j(s,a)$ to demonstrate the variable elimination, because the variable elimination order is only controlled by the scopes $Z^h_j$ indexed by $j$, so procedure is the same for each $c^{(\ell)}_j(s,a)$.

We will use variable elimination to reduce the \eqref{eq:simplified_oracle_obj} to a tractable linear program. 
Let some order criterion $\mathcal{O}$ over $1 \dots m$ be given, where $\mathcal{O}(k)$ returns a variable to eliminate at time step $k = 1\dots m$. Note that determining the optimal order $\mathcal{O}^*$ is in general NP-hard. 
At each iteration of variable elimination, we will bring the relevant state variable $\mathcal{S}_k$ inside the max. 
\textbf{Algorithm \ref{algo:simplify_oracle_obj}} gives the full description of our proposed simplification, heavily based on \cite{delgado} and \cite{guestrin}.

To illustrate the variable elimination procedure, we will work through the hypothetical example from \cite{delgado} while noting differences along the way. Suppose that $\mathcal{O}(1) = \mathcal{S}_1$ at the first iteration of simplification, and that the only scopes $Z_i^R$ and $\text{Pa}(Z_j^h)$ including $\mathcal{S}_1$ are $Z_1^R = \mathcal{S}_1$ and $\text{Pa}(Z_1^h) = \mathcal{S}_1 \times \mathcal{S}_4$. Here, the function $c_1$ is scoped on $\text{Pa}(Z_1^h)$ due to the transition function being backprojected for simplification earlier (see \eqref{eq:backprojection}). Therefore, we can rewrite \eqref{eq:simplified_oracle_obj} as follows due to linearity of the objective.
\begin{equation}
    \label{eq:example_simplification}
    \underset{A}{\max} \bigg[ \sum_{i=2}^l \overline{R}_i(x[Z_i^R]) + \sum_{j=0, 2\dots \phi}  c_j(s, a) + \underset{\mathcal{S}_1, \overline{R}_1 \in \mathcal{R}_t^1, P(\cdot | s[\text{Pa}(Z_1^h)], a) \in \mathcal{P}_t^1}{\max}\Big[ \overline{R}_1(x[Z_1^R]) +  c_1(s,a) \Big] \bigg]
\end{equation}
Where $A$ is as before, but with $\mathcal{S}_1$, $l=1$, and $j=1$ removed: $A = \mathcal{S}_2, \dots, \mathcal{S}_m,, \overline{R}_i \in \mathcal{R}_t^i, P(\cdot | s[\text{Pa}(Z_j^h)], a) \in \mathcal{P}_t^j\  \forall i = 2\dots l\ \forall j =0, 2 \dots \phi$. In general, we will have $L$ relevant functions to pull into the second max each iteration, which we will rename as $u_{Z_1}^{f_1}, \dots, u_{Z_L}^{f_L}$. In our example, we have that $u_{\mathcal{S}_1, a}^{f_1} = \overline{R}_1(x[Z_1^R])$ and $u_{\mathcal{S}_1, \mathcal{S}_4}^{f_2} =  c_1(s,a)$.

For each variable $\mathcal{S}_k$ we wish to eliminate, we select the $L$ relevant functions and replace them with a maximization over $\mathcal{S}_k$ as follows. Here we diverge from \cite{delgado} since they need only maximize over $\mathcal{S}_k$, but we still have a maximization over $R, P$.
\begin{equation}
    u_Z^{e_r} = \underset{\mathcal{S}_k, \mathcal{R}_t^i, \mathcal{P}_t^j}{\max} \sum_{j=1}^L u_{Z_i}^{f_j}
\end{equation}
Where $Z$ is the union of all variables appearing in any scope $Z_i$ setminus the variable $\mathcal{S}_k$, since we maximize it out. Note that there may be none or any number of relevant reward and marginal distribution functions (within $c$) in a single $u_Z^{e_r}$, and we must include all relevant confidence sets within the maximization. Each confidence set will belong only to the relevant $u_Z^{e_r}$ which is the first to pull it out of the larger max in \eqref{eq:simplified_oracle_obj} according to the elimination order criterion $\mathcal{O}$. Note that $u_Z^{e_r}$ is a \textit{new variable} which we add to the optimization procedure.

For ease of notation, for the factored reward functions we will only refer to the state variables within their scope, since the action must be included in the scope anyways. Returning to the example, our $Z$ will be $\{\mathcal{S}_1\} \cup \{ \mathcal{S}_1, \mathcal{S}_4\} \setminus \{ \mathcal{S}_1 \}$. So we have that
\begin{equation}
    \label{eq:def_u_er}
    u_{\mathcal{S}_4}^{e_r} = \underset{\mathcal{S}_1, \overline{R}_1 \in \mathcal{R}_t^1, P(\cdot | s[\text{Pa}(Z_1^h)], a) \in \mathcal{P}_t^1}{\max}\Big[ u_{\mathcal{S}_1}^{f_1} + u_{\mathcal{S}_1, \mathcal{S}_4}^{f_2} \Big],
\end{equation}
and we can then rewrite \eqref{eq:example_simplification} as
\begin{equation}
    \underset{A}{\max} \bigg[ \sum_{i=2}^l \overline{R}_i(x[Z_i^R]) + \sum_{j=0, 2\dots \phi} c_j(s, a) + u_{\mathcal{S}_4}^{e_r} \bigg],
\end{equation}
with $A = \mathcal{S}_2, \dots, \mathcal{S}_m,, \overline{R}_i \in \mathcal{R}_t^i, P(\cdot | s[\text{Pa}(Z_j^h)], a) \in \mathcal{P}_t^j\  \forall i = 2\dots l\ \forall j =0, 2 \dots \phi$. However, to enforce the definition of $u_{\mathcal{S}_1}^{e_r}$ in \eqref{eq:def_u_er}, we need four new inequality constraints, one for each combination of $\mathcal{S}_1$ and $\mathcal{S}_4$ (in the binary state variable case):
\begin{align}
    u_{s_4}^{e_r} &\geq u_{s_1}^{f_1} + u_{s_1, s_4}^{f_2}, \\
    u_{\overline{s_4}}^{e_r} &\geq u_{s_1}^{f_1} + u_{s_1, \overline{s_4}}^{f_2}, \\
    u_{s_4}^{e_r} &\geq u_{\overline{s_1}}^{f_1} + u_{\overline{s_1}, s_4}^{f_2}, \\
    u_{\overline{s_4}}^{e_r} &\geq u_{\overline{s_1}}^{f_1} + u_{\overline{s_1}, \overline{s_4}}^{f_2}.
\end{align}
Furthermore, we need to also consider the relevant confidence sets $\mathcal{R}_t^1$ and $\mathcal{P}_t^1$. For example, consider $u_{\overline{s_1}}^{f_1} = \underset{\mathcal{R}_t^1}{\max}\  \overline{R}_1(\overline{s_1}, a)$. The appropriate confidence set $\mathcal{R}_t^1$ has width based on how many times the pair $\overline{s_1}, a$ has been observed up to time $t$. Note that $\overline{s_1}$ here refers only to the value of the first state variable in the state vector (which is set to zero), the rest of the state values are arbitrary. However, since we are maximizing we can exactly set $\overline{R}_1(\overline{s_1}, a)$ to the maximum value in the confidence set given by:
\begin{align}
     \overline{R}_1(\overline{s_1}, a) &= \hat{f_t}(\overline{s_1}, a) + \sqrt{\frac{d_t}{n_t(\overline{s_1}, a)}} \\
     &= \frac{1}{n_t(\overline{s_1}, a)} \sum_{\tau < t; x_\tau = x} \delta{y_\tau} + \sqrt{\frac{d_t}{n_t(\overline{s_1}, a)}}
\end{align} in $O(1)$ time. In general, we can compute $\overline{R}_i$ for any assignment $z \in \text{Val}(Z_i^R)$ in $O(1)$ time as follows:
\begin{equation}
    \label{eq:suppl_reward_fn_opt}
    \overline{R}_i(z) = \frac{1}{n_t(z)} \sum_{\tau < t; x_\tau = x} \delta{y_\tau} + \sqrt{\frac{d_t}{n_t(z)}}
\end{equation}
Similarly, we must optimize for each assignment $z \in \text{Val(Pa}(Z_j^h))$, for example, $u_{s_1, s_4}^{f_2} = \underset{\mathcal{P}_t^1}{\max}\ c_1(s_1, s_4, a)$, where $s_1=1$ and $s_4=1$ is given. We can optimize for $c_j$ w.r.t some assignment $z$ by \textbf{Algorithm}~\ref{algo:trans_fn_opt}, similar to Figure 2 of \cite{jaksch} and originally given by \cite{strehl2008}. A full proof is given in \cite{jaksch}.

\begin{algorithm}[tb]
% \SetAlgoLined
% \small
\caption{Transition Function Optimization}
\label{algo:trans_fn_opt}
\begin{algorithmic}
\STATE Optimal marginal transition function $P^{(\ell)}(\cdot | z,a)$ is returned for some assignment $z \in \text{Val(Pa}(Z_j^h))$.
\STATE \textbf{Sort} $S = \text{Val}(Z_j^h) = \{s_1', \dots, s_k' \}$ in descending order s.t. $h_j(s_1') \geq \dots \geq h_j(s_k')$. Reverse  order if $w^{(\ell)}_j < 0$. 
\STATE \textbf{Set} $P^{(\ell)}(s_1' | z, a) \coloneqq \min \{1, \hat{P}(s_1' | z,a) + \frac{1}{2}\sqrt{\frac{d_t}{n_t(z,a)}} \}$ 
\STATE \textbf{Set} $P^{(\ell)}(s_j' | z, a) \coloneqq \hat{P}(s_j' | z, a) \text{ for all states } s_j' \text{ s.t. } j > 1.$
\STATE \textbf{Set} $i \coloneqq k$ 
\WHILE{$\sum_{s_j' \in S} P^{(\ell)}(s_j') > 1$}
    \STATE \textbf{Reset} $P^{(\ell)}(s_i'|z, a) \coloneqq \max \{ 0, 1 - \sum_{s_j' \neq s_i'} P^{(\ell)}(s_j' | z, a)\}$ 
    \STATE \textbf{Set} $i \coloneqq i - 1$
\ENDWHILE
\end{algorithmic}
\end{algorithm}

% \begin{algorithm}
% \SetAlgoLined
% \small
% \KwResult{Optimal marginal transition function $P(\cdot | z,a)$ is returned for some assignment $z \in \text{Val(Pa}(Z_j^h))$.}
% \textbf{Sort} $S = \text{Val}(Z_j^h) = \{s_1', \dots, s_k' \}$ in descending order s.t. $h_j(s_1') \geq \dots \geq h_j(s_k')$. Reverse  order if $w_j < 0$. \\
% \textbf{Set}
% %\begin{align*}
%     $P(s_1' | z, a) \coloneqq \min \{1, \hat{P}(s_1' | z,a) + \frac{1}{2}\sqrt{\frac{d_t}{n_t(z,a)}} \}$ \\
% \textbf{Set}
%     $P(s_j' | z, a) \coloneqq \hat{P}(s_j' | z, a) \text{ for all states } s_j' \text{ s.t. } j > 1.$\\
% %\end{align*} \\
% \textbf{Set} $\ell \coloneqq k$ \\
% \While {$\sum_{s_j' \in S} P(s_j') > 1$} {
%     \textbf{Reset} $P(s_\ell'|z, a) \coloneqq \max \{ 0, 1 - \sum_{s_j' \neq s_\ell'} P(s_j' | z, a)\}$ \\
%     \textbf{Set} $\ell \coloneqq \ell - 1$
% }
% \caption{Transition Function Optimization}
% \label{algo:trans_fn_opt}
% \end{algorithm}

% \bigskip

\begin{lemma}
\label{lemma:remove_bounded_MDP}
For all $\mathbf{w}$, we can precompute each function $\overline{R}_i$ and $c_j$ to remove the bounded nature of our MDP in polynomial time.
\end{lemma}
\begin{proof}
Let $\mathbf{w}$ be fixed and given. Assume that some $\overline{R_i}$ has restricted scope $Z_i^R$. For a given $s,a$ pair, we know that $\overline{R}_i \in \mathcal{R}_t^i$ since the width of the confidence set $\mathcal{R}_t^i$ depends on the $s,a$ pair scoped on $Z_i^R$. However, the scope $Z_i^R$ can take only a polynomial number of different assignments. Therefore, we can iterate over all assignments $z \in \text{Val}(Z_i^R)$ and compute the maximum $\overline{R}_i$ for each. Since $\overline{R}_i$ is a single dimensional value, the maximum takes exactly the form \eqref{eq:reward_fn_opt}.

We can do a similar procedure for each $c_j$, which is scoped on $\text{Pa}(Z_j^h)$, although optimization here is multidimensional. By iterating over all $\text{Val}(\text{Pa}(Z_j^h))$, we can solve the optimization problem given by \eqref{eq:trans_fn_opt} independently for both possible signs of $\mathbf{w}$.

Since the number of confidence sets is polynomial, and solving over each is a polynomial time operation, we can remove the ``imprecise" nature of our MDP in polynomial time by explicitly optimizing for the transition and reward functions.
\end{proof}

\begin{algorithm}
\caption{Separation Oracle Objective Simplification}
\begin{algorithmic}
\label{algo:simplify_oracle_obj}
\STATE Optimal objective value \eqref{eq:simplified_oracle_obj} for a fixed action $a$ is returned.
\STATE // \textit{Data structure for constraints of LP\\}
\STATE \textbf{Let} $\Omega = \{\}$ \;
\STATE // \textit{Data structure for functions generated by variable elimination\\}
\STATE \textbf{Let} $\mathcal{F} = \{ \}$ \;
\STATE // \textit{Generate equality constraints using lookup over pre-computed confidence set values \\}
\FOR{$j = 1 \dots \phi$}{
    \FOR{each assignment $z \in \text{Val(Pa}(Z_j^h))$}{
    \STATE Create a new LP variable $u_z^{f_j}$ and add the constraint to $\Omega$:
    \begin{equation*}
        u_z^{f_j} = \underset{\mathcal{P}_t^j}{\max}\ c_j (z, a)
    \end{equation*}
    \STATE Plug in RHS from lookup table generated by \textbf{Algorithm}~\ref{algo:trans_fn_opt}. \\
    \STATE \textbf{Store} new function $f_j$ to be used in variable elimination step: $\mathcal{F} = \mathcal{F} \cup \{ f_j \}$.
    }
    \ENDFOR
 }
\ENDFOR
\FOR{$i=1 \dots l$}{
    \FOR{each assignment $z \in \text{Val}(Z_i^R)$}{
    \STATE Create a new LP variable $u_z^{f_i}$ and add the constraint to $\Omega$:
    \begin{equation*}
        u_z^{f_i} = \underset{\mathcal{R}_t^i}{\max}\ \overline{R}_i (z, a)
    \end{equation*}
    \STATE Plug in RHS from lookup table generated by \eqref{eq:reward_fn_opt}. \\
    \STATE \textbf{Store} new function $f_i$ to be used in variable elimination step: $\mathcal{F} = \mathcal{F} \cup \{ f_i \}$.
    }
    \ENDFOR
}
\ENDFOR

// \textit{Now, $\mathcal{F}$ and $\Omega$ contain all the functions and constraints we need to construct the simplified objective using variable elimination.\\}
\FOR{$i=1 \dots m$}{
    \STATE // \textit{Next variable to be eliminated\\}
    \STATE \textbf{Let} $l = \mathcal{O}(i)$ \;
    \STATE // \textit{Select the relevant functions from $\mathcal{F}$\\}
    \STATE \textbf{Let} $e_1, \dots, e_L$ be the functions in $\mathcal{F}$ whose scope contains $\mathcal{S}_l$, and let $Z_j = \textit{Scope}[e_j]$.\\
    \STATE // \textit{Introduce linear constraints for maximum over current variable $\mathcal{S}_l$\\}
    \STATE \textbf{Define} A new function $e$ with scope $Z = \cup_{j=1}^L Z_j - \{ \mathcal{S}_l \}$ to represent $\max_{s_l} \sum_{j=1}^L e_j$.\\
    \STATE // \textit{Add constraints $\Omega$ to enforce maximum.\\}
    \FOR{each assignment $z \in \text{Val}(Z)$}{
        \STATE \textbf{Add} constraints to $\Omega$ to enforce max:
        \begin{equation*}
            u_z^e \geq \sum_{j=1}^L u_{(z, s_l)[Z_j]}^{e_j} \quad \forall s_l 
        \end{equation*}
    }
    \ENDFOR
    \STATE // \textit{Update set of functions.\\} 
    \STATE $\mathcal{F} = \mathcal{F} \cup \{ e \} \setminus \{ e_1, \dots, e_L\}$
}
\ENDFOR
\STATE // \textit{Now, all variables have been eliminated and all functions have empty scope.\\}
\STATE \textbf{Let} $\kappa$ be the objective value at the solution of the following LP:
\begin{equation}
    \label{eq:obj_minimization}
    \begin{aligned}
        \min_{j=1 \dots |\mathcal{F}|} \quad & \sum_{e_j \in \mathcal{F}} u_{z_j}^{e_j} \\
        \textrm{s.t.} \quad & \Omega \\
    \end{aligned}
\end{equation}
\STATE \textbf{Return} $\kappa$.
\end{algorithmic}
\end{algorithm}

\subsection{Separation Oracle Proofs}
\label{appx:oracle_proof}
We will prove that this reduction is tight, and that we can extract a state $s$ where the constraint is violated if $\mathbf{w}$ lies outside the feasible set.

\begin{lemma}
\label{lemma:omega}
Minimizing \eqref{eq:obj_minimization} will return a polynomial sized set of tight constraints $\omega \subseteq \Omega$ if $\kappa > 0$ where $\kappa$ is the objective value at the solution of the LP in \eqref{eq:obj_minimization}.
\end{lemma}
\begin{proof}
Due to \textbf{Lemma}~\ref{lemma:remove_bounded_MDP}, the only difference between our algorithm and \cite{guestrin} is that instead of adding \eqref{eq:obj_minimization} as a constraint relative to $\kappa$, we explicitly minimize over it. Once we retrieve its minimum objective value, we compare that to $0$. If it is less than or equal to $0$, then our current $\mathbf{w}$ belongs in the feasible set, i.e. it satisfies the exponentially many constraints of our program by setting $\phi = 0$ in the induction proof of Theorem 4.4 of \cite{guestrin}. This follows from enforcing that each introduced variable must satisfy being at least as large as the sum of the relevant functions it represents.

Now assume that $\kappa > 0$. By minimization of a sum of LP variables, each $u_{z_j}^{e_j}$ must be tight on at least one constraint by construction, given by an assignment to some subset of variables. Add this constraint to $\omega$ for each $j = 1 \dots |\mathcal{F}|$. Since $|\Omega|$ is poly$(m)$ by \cite{guestrin}, so is $\omega \subset \Omega$.
\end{proof}

A \textit{strong} oracle is an oracle which returns either the point given to it if the point lies in the solution set, or a separating halfspace / hyperplane which completely contains the feasible solution set and does not contain the query point.

We restate \textbf{Thm.~}\ref{them:oracle_poly} from the main text, and provide a proof:

\textbf{Theorem~\ref{them:oracle_poly}.} Given an efficient variable elimination ordering over the induced cost network, a polynomial-time (strong) separation oracle exists.
% \begin{theorem}
% \label{thm:oracle_poly}
% We can construct a (strong) separation oracle in polynomial time.
% \end{theorem}
\begin{proof}
\label{proof:oracle_poly}
For each action $a$, run \textbf{Algorithm}~\ref{algo:simplify_oracle_obj}. Take the maximum objective value $\kappa^*$ of \eqref{eq:obj_minimization} over all actions $a$. If $\kappa^* \leq 0$, then $\mathbf{w}$ lies in the set described by the exponential number of state constraints. If $\kappa^* > 0$, then we have a set of tight constraints $\omega$ given by \textbf{Lemma}~\ref{lemma:omega}, since $\kappa^*$ is exactly the $\kappa$ for some action $a$. Any state $s = (s_1, \dots, s_m)$ which is consistent with assignments within the tight constraints $\omega$ will be a violating constraint in \eqref{fig:lp_factored}. This is due to the fact that the simplified tight constraint, when $\kappa^* > 0$, represents an $s,a$ constraint violation in the original formulation \eqref{eq:simplified_oracle_obj}. 
% Then we can use this violated constraint to construct a hyper-plane that separates the infeasible $\mathbf{w}$ and the feasible region.
% This is proven again by induction in Theorem 4.4 of \cite{guestrin}.

We can then use the $s, a$ and appropriately optimize for each $\overline{R}_i$ and $P$ marginal described by this violating constraint as a separating hyperplane in terms of $\mathbf{w}$ as follows:
\begin{equation}
     hp(\mathbf{w}) 
     = \sum_{i=1}^l \overline{R}_i(s,a) 
     + \sum_{j=0}^\phi  %\gamma w_j 
     \left(
     - w^{(\ell)}_j h_j(s) 
        + \sum_{\hat{s}' \in \text{Val}(Z_j^h)} 
        w^{(\ell+1)}_j h_j(\hat{s}') 
        P^{(\ell+1)}\left(\hat{s}'| s[\text{Pa}(Z_j^h)], a\right)
     \right)
\end{equation}
\end{proof}

\subsection{Convergence of Ellipsoid Method}
\label{appx:ellipsoid_convergence}
% \textbf{Theorem~\ref{them:planner_result}}.~
\textbf{Theorem~\ref{them:planner_result}.} The Ellipsoid algorithm solves
the optimization problem \textbf{Fig}.~\ref{fig:lp_factored}
in polynomial time.

\begin{proof}
By \textbf{Theorem~6.4.9} of \cite{ellipsoid_book}, the strong optimization problem of maximizing $c^T \mathbf{w}$ over some convex set $P$ (which may require asserting that $P$ is empty) can be solved given a strong separation oracle. However, the optimization problem must be over a ``well-described polyhedron", $P$. By definition, $P$ is well described if there exists a system of inequalities with rational coefficients that has a solution set $P$ such that the encoding length of each inequality in the system is at most $\gamma$ (\textbf{Definition~6.2.2} \cite{ellipsoid_book}). 

Although our system is defined by an exponential number of state constraints \eqref{fig:lp_factored}, at the solution to the problem each reward and transition marginal function is fixed. Therefore, we can represent each inequality in binary with some bounded length $\gamma$.

We also have a strong separation oracle by \textbf{Theorem~\ref{them:oracle_poly}}: an oracle which returns either the point $\mathbf{w}_t$ if given a point in $P$ or a separating hyperplane completely containing $P$. Lastly, to apply the ellipsoid algorithm to strong optimization in polynomial time, one binary searches for the minimum objective value $d$ by solving a sequence of ellipsoid problems with $c^T \mathbf{w} \leq d_t$ added to the inequality set $P$. This also has bounded encoding length. Therefore, our polyhedron $P$ is well-described, and we can solve the strong optimization problem in polynomial time.
\end{proof}

\section{Full Regret Analysis}
\label{appx:regret_analysis}
Our analysis closely follows \citet{osband}.
% Here we mainly present the part of the analysis where 
The main difference is that we do not use the product transition structure as in \citet{osband} and instead use the linear basis scopes of the $V$ function. 
% We define the following step-wise Bellman operator: 
% \begin{definition}
% \label{def:bellman}
% The Bellman operator $\mathcal{T}^M_{\mu^M, i}$ for the MDP $M$, policy $\mu$, and value function $V_i: \mathcal{S} \to \mathbb{R}, \forall i \in [\tau]$, is defined by
% %\begin{align*}
% %    \mathcal{T}_\mu^M V(s) \coloneqq &
%     $
%     \mathcal{T}^M_{\mu^M,i} V_i(s) = 
%     \overline{R}^M(s, \mu(s))
%      + \sum_{s' \in \mathcal{S}} P^M(s' | s, \mu(s)) V_{i+1}(s').
%      $
% %\end{align*}
% % where $Z_j^h$ is the scope of the $j$th basis function $h_j$. 
% \end{definition}
We begin the full regret analysis of our algorithm.
We simplify our notation by writing $*$ in place of $M^*$ or $\mu^*$, and $k$ in place of $\tilde{M}_k$ and $\tilde{\mu}_k$.
We begin by adding and subtracting the computed optimal reward. Let $s_{t_k + 1}$ be the first state in the $k$th episode. Then the regret at episode $k$ decomposes as follows.
\begin{equation}
    \Delta_k = V_{*,1}^*(s_{t_k + 1}) - V_{k,1}^*(s_{t_k + 1}) = \bigg( 
    V_{k,1}^k(s_{t_k + 1}) - V_{k,1}^*(s_{t_k + 1})
    \bigg)
    +
    \bigg(
    V_{*,1}^*(s_{t_k + 1}) - V_{k,1}^k(s_{t_k + 1})
    \bigg)
\end{equation}

The term $V_{*,1}^*(s_{t_k + 1}) - V_{k,1}^k(s_{t_k + 1})$ relates the optimal rewards of the MDP $M^*$ to those near optimal for $\tilde{M}$. 
We can bound this difference by planning accuracy $\epsilon = \sqrt{1 / k}$ by optimism. 
Indeed, any relaxation to $R$ or $P$ can only cause the computed $V^k_{k,1}$ (without planning error) to be larger than the actual $V^*_{*,1}$ because
 the argmax over our relaxed $R$ and $P$ can only make the RHS of \eqref{eq:oracle_obj} larger, which in turn makes the RHS of the inequalities in \textbf{Fig}.~\ref{fig:lp_factored} larger.
Importantly, this includes the relaxation where we don't insist that the transition marginals are consistent (in that they represent the marginals of a real distribution).
This is what allowed us to relax enforcing that the marginals are consistent within our proposed oracle.

$V^k_k$ also overestimates $V^*_k$ because $V^*_k$ is worse than $V^*_*$, which by definition uses the best $\mu^*$ instead of $\mu^k$.

% # Old no \gamma approach
We then decompose the first term by repeated application of the dynamic programming of Bellman operator \cite{osband_dp}:
\begin{equation}
    \label{eq:applying_dp}
    (V_{k,1}^k - V_{k,1}^*)(s_{t_k + 1}) = 
    \sum_{\ell=1}^\tau (\mathcal{T}_{k, \ell}^k - \mathcal{T}_{k, \ell}^*)V_{k, \ell+1}^k (s_{t_k + \ell}) + \sum_{\ell=1}^{\tau} d_{t_k + \ell},
\end{equation}
where 
$d_{t_k + \ell} \coloneqq 
\sum_{s \in \mathcal{S}} 
\bigg\{ P^*(s | x_{k,\ell}) (V_{k, \ell+1}^k - V_{k, \ell+1}^* )(s)\bigg\} 
- (V_{k, \ell+1}^k - V_{k, \ell+1}^*)(s_{t_k + \ell+1})$, 
and $x_{k, \ell} = (s_{t_k + \ell}, \mu_k(s_{t_k + \ell})) $.
The derivation is as follows:
\begin{align}
\small
    (V_{k,1}^k - V_{k,1}^*)(s_{t_k + 1}) 
    =& 
    \left( \mathcal{T}_{k, 1}^k V_{k, 2}^k - \mathcal{T}_{k, 1}^* V_{k, 2}^* \right) (s_{t_k+1}) 
    \nonumber \\
    =& \left( \mathcal{T}_{k, 1}^k V_{k, 2}^k - \mathcal{T}_{k, 1}^* V_{k, 2}^k 
+ \mathcal{T}_{k, 1}^* V_{k, 2}^k  - \mathcal{T}_{k, 1}^* V_{k, 2}^* \right) (s_{t_k+1}) 
\nonumber \\
    =& \left[  \left(\mathcal{T}_{k, 1}^k - \mathcal{T}_{k, 1}^* \right)  V_{k, 2}^k
    +  \mathcal{T}_{k, 1}^* \left( V_{k, 2}^k - V_{k, 2}^* \right) 
    \right](s_{t_k+1})
\nonumber \\
    =& \left(\mathcal{T}_{k, 1}^k - \mathcal{T}_{k, 1}^* \right) V_{k, 2}^k(s_{t_k+1})
    + \sum_{s' \in \mathcal{S}}P^*(s'|x_{k,1})\left( V_{k, 2}^k - V_{k, 2}^* \right)(s'), 
\nonumber 
\end{align}
where 
$\mathcal{T}_{k, 1}^* \left( V_{k, 2}^k - V_{k, 2}^* \right) (s_{t_k+1})
= R^*(x_{k,1}) + \sum_{s\in\mathcal{S}}P^*(s'|x_{k,1}) V_{k, 2}^k(s')
- R^*(x_{k,1}) - \sum_{s\in\mathcal{S}}P^*(s'|x_{k,1}) V_{k, 2}^*(s')$.
Continuing the derivation above, we have:
\begin{align}
    =& \left(\mathcal{T}_{k, 1}^k - \mathcal{T}_{k, 1}^* \right) V_{k, 2}^k(s_{t_k+1})
    + \sum_{s' \in \mathcal{S}}P^*(s'|x_{k,1})\left( V_{k, 2}^k - V_{k, 2}^* \right)(s') 
    \nonumber \\
    &- \left( V_{k, 2}^k - V_{k, 2}^* \right) (s_{t_k+2})
    + \left( V_{k, 2}^k - V_{k, 2}^* \right) (s_{t_k+2}) 
\nonumber \\
   =& \left(\mathcal{T}_{k, 1}^k - \mathcal{T}_{k, 1}^* \right) V_{k, 2}^k(s_{t_k+1}) + d_{t_k+1} 
   + \left( V_{k, 2}^k - V_{k, 2}^* \right) (s_{t_k+2}) 
\nonumber \\
   =& \left(\mathcal{T}_{k, 1}^k - \mathcal{T}_{k, 1}^* \right) V_{k, 2}^k(s_{t_k+1}) + d_{t_k+1} 
   + \left( \mathcal{T}_{k, 2}^k V_{k, 3}^k - \mathcal{T}_{k, 2}^* V_{k, 3}^* \right) (s_{t_k+2}) 
   \nonumber \\
   =& \ \dots 
   \nonumber \\
   =&
   \sum_{\ell=1}^\tau (\mathcal{T}_{k, \ell}^k - \mathcal{T}_{k, \ell}^*)V_{k, \ell+1}^k (s_{t_k + \ell}) + \sum_{\ell=1}^{\tau} d_{t_k + \ell}.
   \nonumber
\end{align}
Note that we can apply $V_{k,\ell}^* = \mathcal{T}_{k, \ell}^* V_{k,\ell+1}^*$ because here we are applying the action of $\mu^k$ to the actual environment of $M^*$, and 
$V_{k,\ell}^k =   \mathcal{T}_{k, \ell}^k V_{k,\ell+1}^k$ because at the optimal solution, the LP constraints in \textbf{Fig}.~\ref{fig:lp_factored} are tight:
% \iffalse{
%%% NOTE: I think in later Cauchy-Schwartz, it's supposed to be w^{\ell+1} along with P^{\ell+1}, not w^{\ell+1} along with P^{\ell}
\begin{align}
V_{k,\ell}^k (s_{t_k + \ell})
=&      
\sum_{j=0}^\phi w^{k,(\ell)}_{k,j} h_j(s_{t_k + \ell})
\nonumber \\
= & 
\sum_{i=1}^l \overline{R}^k_i(s_{t_k + \ell},\mu^k(s_{t_k + \ell})) 
+  \sum_{j=0}^\phi \sum_{\hat{s}' \in \text{Val}(Z_j^h)} \!\!\!\!\!\!
    w^{k,(\ell+1)}_{k,j} h_j(\hat{s}') 
    P^{k,(\ell+1)}_j(\hat{s}'| s_{t_k + \ell}[\text{Pa}(Z_j^h)], \mu^k(s_{t_k + \ell}))
\nonumber \\
= & 
\overline{R}^k(s_{t_k + \ell},\mu^k(s_{t_k + \ell})) 
+ 
\sum_{s' \in \mathcal{S}} P^{k,(\ell+1)}(s' | x_{k,\ell}) V_{k,\ell+1}^k (s')
\nonumber \\
=& 
\mathcal{T}_{k, \ell}^k V_{k,\ell+1}^k (s_{t_k + \ell}).
\nonumber
\end{align}
\begin{lemma}
\label{lemma:martingale}
The quantity $d_{t_k}$ is a bounded martingale difference.
\end{lemma}
\begin{proof}
% Let $(s_{t_k + 1}, \mu_k(s_{t_k + 1})) = x_{k, 1}$, then we have that:
\begin{align}
    \mathbb{E}[ d_{t_k+\ell} ] &= 
    \mathbb{E}\bigg[  \sum_{s \in \mathcal{S}} \bigg\{ P^*(s | x_{k, \ell}) (V_{k,\ell+1}^k  - V_{k,\ell+1}^*)(s)\bigg\} \bigg] 
    - \mathbb{E} \bigg[  (V_{k,\ell+1}^k 
    - V_{k,\ell+1}^*)(s_{t_k + \ell+1}) \bigg] \\
    &=  
    \bigg[ 
         \sum_{s \in \mathcal{S}} \bigg\{ P^*(s | x_{k, \ell}) (V_{k,\ell+1}^k - V_{k,\ell+1}^*)(s)\bigg\} 
    \bigg] - 
    \bigg[ 
         \sum_{s \in \mathcal{S}} \bigg\{ P^*(s | x_{k, \ell}) (V_{k,\ell+1}^k - V_{k,\ell+1}^*)(s)\bigg\}
    \bigg] = 0,
\end{align}
since the first term already takes the expectation, so $d_{t_k+\ell}$ is a martingale difference. 
Furthermore, we can show that is bounded as follows.
\begin{align}
    d_{t_k+\ell} &= \sum_{s \in \mathcal{S}} \bigg\{ P^*(s | x_{k, \ell}) (V_{k,\ell+1}^k - V_{k,\ell+1}^*)(s)\bigg\} -  
    (V_{k,\ell+1}^k - V_{k,\ell+1}^*)(s_{t_k + \ell+1}) \\
    &\leq \sum_{s \in \mathcal{S}} \bigg\{ P^*(s | x_{k, \ell}) (V_{k,\ell+1}^k - V_{k,\ell+1}^*)(s)\bigg\} \\
    \label{eq:V_k_k_i_V_star_k_i}
    &\leq \max_{s \in \mathcal{S}} (V_{k,\ell+1}^k - V_{k,\ell+1}^*)(s) \\
    % \label{eq:V_k_k_i_V_star_k_i}
    % &\leq CD, 
    \label{eq:V_k_k_i_V_star_k_i_next}
    &\leq \max_{s \in \mathcal{S}} V^k_{k,\ell+1}(s) 
    \leq 
    \max_{s \in \mathcal{S}} \left|\sum_{j=1}^\phi w^{(\ell+1)}_j h_j(s)\right| \\
    % \label{eq:V_k_k_i_V_star_k_i_next}
    & \leq 
    \| \mathbf{w} \|_1
    \max_{s \in \mathcal{S}} \max_j | h_j(s) |
    = \| \mathbf{w} \|_1
    \max_j 
    \max_{s \in \text{Val}(Z_j^h)} |h_j (s)| 
\end{align}
The last fact is proven by H\"older's inequality. Note that in this analysis we do not use or assume a factored linear expansion of $V^*_{k, \ell+1}$.
\end{proof}

Importantly, the above bound is not dependent on the diameter of the MDP, which may be exponential in general. 
With a bounded martingale difference, we may then use the Azuma-Hoeffding inequality to obtain the following concentration guarantee \cite{osband}, \cite{jaksch}:
\begin{equation}
    \label{eq:azuma_application}
    \mathbb{P}\bigg( 
    \sum_{k=1}^{\lceil T / \tau \rceil} 
    \sum_{\ell=1}^{\tau}
    d_{t_k+\ell} > \| \mathbf{w} \|_1
    \max_j 
    \max_{s \in \text{Val}(Z_j^h)} |h_j (s)| 
    % \sqrt{2 \lceil T / \tau \rceil \log( 2 / \delta )}
    \sqrt{2  T \log( 2 / \delta )}
    \bigg) \leq \delta.
\end{equation}

The remaining first term of the RHS of \eqref{eq:applying_dp} is the one step Bellman error of the imagined MDP $\tilde{M}_k$, which depends only on observed states and actions $x_{k, \ell}$. Using Cauchy-Schwartz repeatedly we have the following.
\begin{align}
    \label{eq:T_k_i_T_star_k_i}
    &\mathrel{\phantom{=}} 
    \sum^{\tau}_{\ell=1}
    (\mathcal{T}_{k,\ell}^k - \mathcal{T}_{k,\ell}^*) 
    V_{k,\ell+1}^k (s_{t_k + \ell}) \\
    % \label{eq:T_k_i_T_star_k_i}
    &= \sum^{\tau}_{\ell=1} 
    (\mathcal{T}_{k,\ell}^k - \mathcal{T}_{k,\ell}^*) 
    \sum_{j=1}^\phi w_{k,j}^{k, (\ell+1)} h_{j}(s_{t_k + \ell}) \\ 
    &=  \sum^{\tau}_{\ell=1}
    \left[ (\overline{R}^k(x_{k, \ell}) - \overline{R}^*(x_{k, \ell}))
    + \sum_{s' \in \mathcal{S}} P^{k,(\ell+1)} (s' | x_{k, \ell}) \sum_{j=1}^\phi w_{k,j}^{k, (\ell+1)} h_{j}(s') 
    - \sum_{s' \in \mathcal{S}} P^* (s' | x_{k, \ell}) \sum_{j=1}^\phi w_{k,j}^{k, (\ell+1)} h_{j}(s')
    \right]\\
    &\leq 
    \sum^{\tau}_{\ell=1}
    \bigg[ 
    |\overline{R}^k(x_{k, \ell}) - \overline{R}^*(x_{k, \ell})| 
    +  \sum_{j=1}^\phi \bigg| w_{k,j}^{k, (\ell+1)} \sum_{s' \in \mathcal{S}} (P^{k,(\ell+1)} (s' | x_{k, \ell}) -  P^* (s' | x_{k, \ell})) h_{j}(s') \bigg|\ \bigg] \label{eq:analysis_target}
\end{align}
Note that LHS of \textbf{Eq.}~\eqref{eq:T_k_i_T_star_k_i} does not contain $V^*_{k,\ell}$, so we don't need it to be factored linear either.
Since $x_{k,\ell} = (s_{t_k + \ell}, \mu_k(s_{t_k + \ell}))$ we can simplify further. Denote $\mu_k(s_{t_k + \ell})$ as $a_{k, \ell}$ and we have the following for the rightmost transition function term by H\"older's inequality.
\begin{align}
    \label{eq:transition_bound}
     &\mathrel{\phantom{=}} 
     \sum_{j=1}^\phi \bigg| w_{k,j}^{k, (\ell)}
     \sum_{s' \in \mathcal{S}} 
     (P^{k,(\ell)} (s' | x_{k, \ell}) -  P^* (s' | x_{k, \ell})) h_{j}(s') \bigg| \\
    &=\sum_{j=1}^\phi \bigg| w_{k,j}^{k, (\ell)}
    \sum_{s' \in \text{Val}(Z_j^h)} 
    (P^{k,(\ell)} (s' | s_{t_k + \ell}[\text{Pa}(Z_j^h)], a_{k, \ell}) 
    -  P^* (s' | s_{t_k + \ell}[\text{Pa}(Z_j^h)], a_{k, \ell})) h_{j}(s') \bigg| \\
    &\leq ||\mathbf{w}_{k}^k ||_1 
    \max_j \left| \sum_{s' \in \text{Val}(Z_j^h)} 
    (P^{k,(\ell)} (s' | s_{t_k + \ell}[\text{Pa}(Z_j^h)], a_{k, \ell}) 
    -  P^* (s' | s_{t_k + \ell}[\text{Pa}(Z_j^h)], a_{k, \ell})) h_{j}(s') \right| \\
    &\leq ||\mathbf{w}_{k}^k ||_1 \max_j 
    \bigg[ \max_{s' \in \text{Val}(Z_j^h)} 
    \bigg(|h_{j}(s')|\bigg) 
    \| P^{k,(\ell)} (\cdot | s_{t_k + \ell}[\text{Pa}(Z_j^h)], a_{k, \ell}) -  P^* (\cdot | s_{t_k + \ell}[\text{Pa}(Z_j^h)], a_{k, \ell})) \|_1 
    \bigg]\label{eq:P_fraction_bounds}
\end{align}
This shows that the one step Bellman error is bounded by the diameter of our convex set for $\mathbf{w}$ and a maximum over all basis function transition confidence set accuracy products. Finally, we can also bound the reward function term factor by factor by the triangle inequality:
\begin{align}
    \label{eq:reward_bound}
     &\mathrel{\phantom{=}}  
     |\overline{R}^k(x_{k,\ell}) - \overline{R}_i^*(x_{k,\ell})| \\
     &= | \sum_{i=1}^l 
     \overline{R}_i^k(x_{k,\ell}) - \overline{R}_i^*(x_{k,\ell}) |\\
     &\leq \sum_{i=1}^l |
     \overline{R}_i^k(x_{k,\ell}[Z_i^R]) - \overline{R}_i^*(x_{k,\ell}[Z_i^R])|.
     \label{eq:R_fraction_bounds}
\end{align}
Note that $\|P^k - P^* \|_1$ and $\|R^k - R^* \|_1$ can all be bounded due to the concentration guarantees for the confidence sets. 

\subsection{Concentration Guarantees}
We will use the guarantees provided by \cite{osband}. 
\begin{lemma}
\label{lemma:trans_fn_concentration}
For all finite sets $\mathcal{X}$, finite sets $\mathcal{Y}$, function classes $\mathcal{P} \subseteq \mathcal{P}_{\mathcal{X}, \mathcal{Y}}$, then for any $x \in \mathcal{X}$, $\epsilon > 0$ the deviation of the true distribution $P^*$ to the empirical estimate after $t$ samples $\hat{P}_t$ is bounded:
\begin{equation}
    \mathbb{P}(\|P^*(x) - \hat{P}_t(x) \|_1 \geq \epsilon)
    \leq
    \exp \bigg( |\mathcal{Y}| \log(2) - \frac{n_t(x) \epsilon^2}{2} \bigg)
\end{equation}
\end{lemma}
\begin{proof}
\cite{osband} claims that this is a relaxation of a proof by \cite{weissman}.
\end{proof}

One can show \textbf{Lemma~\ref{lemma:trans_fn_concentration}} ensures that for any $x \in \mathcal{X}$ $\mathbb{P}\bigg(\|P_j^*(x) - \hat{P}_{j_t}(x) \|_1 \geq \sqrt{\frac{2 |\text{Val}(Z_j^h)|\log(2) - 2 \log(\delta')}{n_t(x)}})\bigg) \leq \delta'$. 
Note that previous analysis in \cite{osband} has a minor technical error which changes the choice of $\epsilon$ (\textbf{Appendix~\ref{appendix:osband_error}}).

The number of marginal transition function confidence sets that we have is given by $N = |\mathcal{A}| \sum_{j=1}^\phi |\text{Val(Pa}[Z_j^h])|$. Let us give them some ordering $i \in [N]$. 
Then we define a sequence for each confidence set at each episode $d_{t_k}^{P_j} = 2 |\text{Val}(Z_j^h)|\log(2) - 2 \log(\delta'_{k, i})$, where $\delta'_{k, i} = \delta / (2 N |\text{Pa}[Z_j^h]| k^2 )$. 
Now with a union bound over all confidence set events over all time steps $k$ we have that:
\begin{align}
    \label{eq:transition_fn_union_bound}
    \bigcup_{i=1}^N \bigcup_{k=1}^\infty \mathbb{P}(P_i^* \notin \mathcal{P}_t^i (d_{t_k}^{P_i})) 
    &\leq \sum_{i=1}^N \sum_{k=1}^\infty \delta'_{k, i} 
    = \sum_{i=1}^N \sum_{k=1}^\infty \frac{\delta}{2 N |\text{Pa}[Z_j^h]| k^2} \\
    &= \frac{\delta}{2 N} \frac{\pi^2}{6} \sum_{i=1}^N \frac{1}{|\text{Pa}[Z_j^h]|} 
    \leq \delta \frac{\pi^2}{12} \frac{1}{N} N \leq \delta
\end{align}
So we have that $\mathbb{P}(P_i^* \in \mathcal{P}_t^i (d_{t_k}^{P_i})\ \forall k \in \mathbb{N},\ \forall j \in [N]) \geq 1 - \delta$. 

\begin{lemma}
\label{lemma:reward_fn_concentration}
If $\{ \epsilon_i \}$ are all independent and sub $\sigma$-gaussian, then $\forall \beta \geq 0$:
\begin{equation}
    \mathbb{P}\bigg(
    \frac{1}{n} | \sum_{z=1}^n \epsilon_z| > \beta
    \bigg)
    \leq \exp \bigg(
    \log(2) - \frac{n \beta^2}{2\sigma^2}
    \bigg).
\end{equation}
\end{lemma}
In particular, we may use \textbf{Lemma~\ref{lemma:reward_fn_concentration}} to say that for any $x \in \mathcal{X}$:
\begin{align}
    &\mathbb{P} \bigg( 
    \frac{1}{n_t(x)}| \sum_{z=1}^{n_t(x)} 
    \hat{R}_{i, z}(x) - \overline{R}_i^*(x)| 
    > \sqrt{\frac{\sigma^2 2 \log(\frac{2}{\delta'})}{n_t(x)}} \bigg) \leq \delta'
\end{align} 
Where the sub $\sigma$-gaussian random variable $\hat{R}_{i, z}$ represents the empirical value of the $i$th component of the reward function at the $z$th time the pair $x = (s,a)$ was observed before time $t$. Recall that the true mean of the $i$th reward component $\overline{R}_i^*(x)$ is a fixed scalar value. Now for each component of the factored reward function $i=1 \dots l$, define the sequence $d_{t_k}^{R_i} = \sigma^2 2 \log (2 / \delta'_{k,i})$, where $\delta'_{k,i} = \delta / (2 l | \mathcal{X}[Z_i^R]| k^2)$. With the same union bound as \eqref{eq:transition_fn_union_bound} over all confidence set events over all time steps $k$, we have that:
\begin{align}
    \label{eq:reward_fn_union_bound}
    \bigcup_{i=1}^l \bigcup_{k=1}^\infty \mathbb{P}(\overline{R}_i^* \notin \mathcal{R}_t^i (d_{t_k}^{R_i})) \leq \sum_{i=1}^l \sum_{k=1}^\infty \delta_{k, i} \leq \delta.
\end{align}
Combining \eqref{eq:transition_fn_union_bound} and \eqref{eq:reward_fn_union_bound}, we have that:
\begin{equation}
    \mathbb{P}\bigg(M^* \in \mathcal{M}_k\ \forall k \in \mathbb{N} \bigg) \geq 1 - 2\delta.
\end{equation}

\subsection{Aside: Technical Error in Osband}
\label{appendix:osband_error}
We point out a minor technical error in \cite{osband} which changes the analysis and simplification of the regret. In their \textbf{Section~7.2}, they claim that they may use $\epsilon = \sqrt{\frac{2 |\mathcal{S}_j|}{n_t(x)} \log(\frac{2}{\delta'})}$ with their \textbf{Lemma~2} (our \textbf{Lemma~\ref{lemma:trans_fn_concentration}}) to obtain the following: for any $x \in \mathcal{X}$ $\mathbb{P}\bigg(\|P_j^*(x) - \hat{P}_{j_t}(x) \|_1 \geq \epsilon \bigg) \leq \delta'$. Plugging their choice of $\epsilon$ into \textbf{Lemma~\ref{lemma:trans_fn_concentration}}, we get the following.
\begin{align}
    \mathbb{P}\bigg(\|P_j^*(x) - \hat{P}_{j_t}(x) \|_1 \geq \epsilon \bigg) 
    &\leq \exp \bigg( 
    |\mathcal{Y}| \log(2) 
    - \frac{n_t(x) \epsilon^2}{2} \bigg) \\
    &= \exp \bigg( 
    |\mathcal{S}_j| \log(2) 
    - \frac{n_t(x) \frac{2 |\mathcal{S}_j|}{n_t(x)} \log(\frac{2}{\delta'})}{2} \bigg)\\
    &= \exp \bigg(
    |\mathcal{S}_j| \log(2) 
    - |\mathcal{S}_j| \log(\frac{2}{\delta'})
    \bigg) \\
    &= \exp \bigg(
    |\mathcal{S}_j| \log(\delta') \bigg)
\end{align}
In \cite{osband}, $|\mathcal{S}_j| \in \mathbb{N}$ is the size of the scope for the $j$th transition function. 
In general, $|\mathcal{S}_j| > 1$ implies $\exp \bigg(|\mathcal{S}_j| \log(\delta') \bigg) > \delta'$. 
Therefore, they are \textit{assuming more tightness} than they should with their empirical estimates of the transition functions. 
They subsequently use $d_{t_k}^{P_j} = 2 |\mathcal{S}_j| \log(\frac{2}{\delta'_{k,j}})$ as their increasing sequence, which incorrectly assumes the result above. 

If we wish to end up with $\delta'$, we can solve for the correct $\epsilon$ as follows. 
\begin{align}
    \delta' &= \exp\bigg(|\mathcal{S}_j| \log(2) - \frac{n_t(x)\epsilon^2}{2}\bigg) \\
    \log(\delta') &= |\mathcal{S}_j| \log(2) - \frac{n_t(x)\epsilon^2}{2} \\
    \epsilon &= \sqrt{\frac{ 2|\mathcal{S}_j| \log(2) - 2\log(\delta')}{n_t(x)}}
\end{align}
Now we let $d_{t_k}^{P_j} = 2|\mathcal{S}_j| \log(2) - 2\log(\delta'_{k, j})$, where $\delta'_{k, j} = \delta / (2m |\mathcal{X}[Z_j^P]| k^2)$ which is the same $\delta'_{k,j}$ value as from \cite{osband}. Therefore as $k$ increases, so does $d_{t_k}^{P_j}$, and we still have the increasing sequence required for applications of \textbf{Corollary~\ref{corr:bounded_widths}} from \cite{osband}.

\subsection{Corollary from \cite{osband}}
\begin{corollary}
\label{corr:bounded_widths}
For all finite sets $\mathcal{X}$, measurable spaces $(\mathcal{Y}, \Sigma_{\mathcal{Y}})$, function classes $\mathcal{F} \subseteq \mathcal{M}_{\mathcal{X}, \mathcal{Y}}$ with uniformly bounded widths $w_\mathcal{F} \leq C_\mathcal{F}\ \forall x \in \mathcal{X}$ and non-decreasing sequences $\{ d_t\ :\ t \in \mathcal{N} \}$:
\begin{equation}
    \sum_{k=1}^T w_{\mathcal{F}_k} (x_{t_k + 1}) \leq 4 ( \tau C_\mathcal{F} |\mathcal{X}| + 1 ) + 4 \sqrt{2 d_T |\mathcal{X}| T},
\end{equation}
where $x_{t_k+1}$ is the first $x \in \mathcal{X}$ for episode $k$.
\end{corollary}

\subsection{Regret Bound}
\label{appx:full_regret_thm}
We can now analyze the regret bounds for our algorithm.
\begin{align*}
    &\mathrel{\phantom{=}} \text{Regret}(T, \pi_\tau ,M^*) 
    = \sum_{k=1}^{\lceil T / \tau \rceil} \Delta_k 
    = \sum_{k=1}^{\lceil T / \tau \rceil} \bigg[ 
    \bigg( 
    V_{k,1}^k(s_{t_k + 1}) - V_{k,1}^*(s_{t_k + 1})
    \bigg)
    +
    \bigg(
    V_{*,1}^*(s_{t_k + 1}) - V_{k,1}^k(s_{t_k + 1})
    \bigg)
    \bigg] \\
    &\leq \underbrace{\sum_{k=1}^{\lceil T / \tau \rceil} \bigg[
    \sqrt{1 / k}
    \bigg]
    }_{\circled{1}}
    +
    \underbrace{
    \| \mathbf{w} \|_1
    \max_j 
    \max_{s \in \text{Val}(Z_j^h)} |h_j (s)| \sqrt{2 T \log( 2 / \delta )}}_{\circled{2}} \\
    &+ \underbrace{
    \sum_{k=1}^{\lceil T / \tau \rceil} 
    \sum_{\ell=1}^{\tau}
    \sum_{i=1}^l |\overline{R}_i^k(x_{k,\ell}[Z_i^R]) - \overline{R}_i^*(x_{k,\ell}[Z_i^R])|
    }_{\circled{3}} \\
    &+ \underbrace{
    \sum_{k=1}^{\lceil T / \tau \rceil}
    \sum_{\ell=1}^\tau
    ||\mathbf{w}_{k}^k ||_1 \max_j 
    \bigg[ \max_{s' \in \text{Val}(Z_j^h)} 
    \bigg(|h_{j}(s')|\bigg) 
    \| P^{k,(\ell)} (\cdot | s_{t_k + \ell}[\text{Pa}(Z_j^h)], a_{k, \ell}) -  P^* (\cdot | s_{t_k + \ell}[\text{Pa}(Z_j^h)], a_{k, \ell})) \|_1 
    \bigg]
    }_{\circled{4}}
\end{align*}
With probability at least $1- \delta$ (PAC regret bound), and where $\circled{1}$ is the planning oracle error contribution, $\circled{2}$ is the contribution of the bounded martingale (\textbf{Lemma~\ref{lemma:martingale}}) over all episodes with the Azuma-Hoeffding inequality from \eqref{eq:azuma_application}, $\circled{3}$ is the contribution of the reward functions in the one step Bellman error, and $\circled{4}$ is contribution from the marginal transition functions from \eqref{eq:transition_bound}. 
We begin by bounding $\circled{1} \leq 2 \sqrt{\lceil T / \tau \rceil}$ by integral sum bound. 
Next, let $\max_j \max_{s \in \text{Val}(Z_j^h)} |h_j (s)| \leq G$ be some global bound on all the basis functions which must exist as the value function is bounded over a finite set. Then we can say: $\circled{2} \leq \|\mathbf{w}\|_1 G \sqrt{2 T \log (2/\delta)}$.

Henceforth, let $\lceil T / \tau \rceil = K$ be the number of true episodes. For $\circled{3}$, we apply \textbf{Corollary}~\ref{corr:bounded_widths} and plug in $C_\mathcal{F} = C$ as a width bound of each reward confidence set and $d_T^{R_i}$ as our sequence:
\begin{align}
    \circled{3} 
    &= \sum_{k=1}^K \sum_{\ell=1}^\tau \sum_{i=1}^l 
    |\overline{R}_i^k(x_{k,\ell}[Z_i^R]) - \overline{R}_i^*(x_{k,\ell}[Z_i^R])| \\
    &= \sum_{i=1}^l \bigg[ 
    4(\tau C |\mathcal{X}[Z_i^R]| + 1) 
    + 4 \sqrt{2 \sigma^2 2 \log(2 / (\delta / 2l|\mathcal{X}[Z_i^R]|T^2)) |\mathcal{X}[Z_i^R]| T} 
    \bigg] \\
    &\leq 
    \sum_{i=1}^l \bigg[ 
    5 \tau C |\mathcal{X}[Z_i^R]|
    + 8\sigma \sqrt{ |\mathcal{X}[Z_i^R]| T \log(4 l|\mathcal{X}[Z_i^R]|T^2 / \delta) }
    \bigg]
\end{align}

We can bound the confidence sets of $\circled{4}$ by again applying \textbf{Corollary~\ref{corr:bounded_widths}}.
\begin{align}
    \label{eq:max_j}
    \circled{4} &\leq 
    \|\mathbf{w}\|_1 G
    \sum_{k=1}^K
    \sum_{\ell=1}^\tau
    \max_j 
    \bigg[ 
    \| P^{k,(\ell)} (\cdot | s_{t_k + \ell}[\text{Pa}(Z_j^h)], a_{k, \ell}) -  P^* (\cdot | s_{t_k + \ell}[\text{Pa}(Z_j^h)], a_{k, \ell})) \|_1 
    \bigg] \\
    &\leq \label{eq:sum_j}
    \|\mathbf{w}\|_1 G
    \sum_{k=1}^K
    \sum_{\ell=1}^\tau
    \sum_{j=1}^\phi
    \bigg[ 
    \| P^{k,(\ell)} (\cdot | s_{t_k + \ell}[\text{Pa}(Z_j^h)], a_{k, \ell}) -  P^* (\cdot | s_{t_k + \ell}[\text{Pa}(Z_j^h)], a_{k, \ell})) \|_1 
    \bigg] \\
    &\leq
    \|\mathbf{w}\|_1 G
    \sum_{j=1}^\phi
    \bigg[ 
    4(\tau C_\mathcal{F} |\text{Val}[Z_j^h]| + 1) + 4\sqrt{2| \mathcal{X}[\text{Pa}(Z_j^h)]| T d_T^{P_j} }
    \bigg] \\
    &\leq \|\mathbf{w}\|_1 G
    \sum_{j=1}^\phi
    \bigg[ 
    5 \tau |\text{Val}[Z_j^h]| + 4
    \sqrt{4
    | \mathcal{X}[\text{Pa}(Z_j^h)]| T
    [ |\text{Val}(Z_j^h)|\log(2) - \log(\delta / (2 N |\text{Pa}[Z_j^h]| T^2 ))]
    }
    \bigg]
\end{align}
Where $\phi$ is the number of basis functions, and $d_T^{P_j} =  2 |\text{Val}(Z_j^h)|\log(2) - 2 \log(\delta / (2 N |\text{Pa}[Z_j^h]| T^2 ))$ from our union bound.

\begin{remark}
Note that from \eqref{eq:max_j} to \eqref{eq:sum_j} we do not have a dependence on the number of confidence sets $N$ because we are conditioning on historical state action observations, with respect to individual basis function scopes $Z_j^h$.
\end{remark}

\begin{theorem}
\label{appx_thm:main_regret_bound}
Let $M^*$ be an MDP with our special factored structure as well as an exactly linear factored optimal value function, and an efficient variable elimination ordering $\mathcal{O}$ be given. Using our procedure, we can bound the regret over $T$ iterations ($K$ episodes) for any $M^*$, $\text{Regret}(T, \pi_{\tau}, M^*)$
\begin{align}
    &\leq
    2 \sqrt{K} % circle 1
    +  
    \|\mathbf{w}\|_1 G \sqrt{2 T \log (2/\delta)} % circle 2
    + 
    \sum_{i=1}^l \bigg[ % circle 3
    5 \tau C |\mathcal{X}[Z_i^R]|
    + 8\sigma \sqrt{ |\mathcal{X}[Z_i^R]| T \log(4 l|\mathcal{X}[Z_i^R]|T^2 / \delta) }
    \bigg] \\ 
    &+ % circle 4
    \|\mathbf{w}\|_1 G
    \sum_{j=1}^\phi
    \bigg[ 
    5 \tau |\text{Val}[Z_j^h]| + 4
    \sqrt{4
    | \mathcal{X}[\text{Pa}(Z_j^h)]| T
    [ |\text{Val}(Z_j^h)|\log(2) - \log(\delta / (2 N |\text{Pa}[Z_j^h]| T^2 ))]
    }
    \bigg] \label{eq:finalbound1}
\end{align}
with probability at least $1 - \delta$.
\end{theorem}

We will simplify the bound in the symmetric case similar to \cite{osband} to present our result from the main paper.

\textbf{Theorem~\ref{corr:regret_bound}.}~
% \label{appx_thm:simplified}
Let $l + 1 \leq \phi$, $C = \sigma = 1$, $|\mathcal{S}_i| = |\mathcal{X}_i| = \kappa$, $|Z_i^R| = |\text{Pa}(Z_i^h)| = \zeta$ for all $i$, and let $J = \kappa^\zeta$, and $\| \mathbf{w} \|_1 \leq W$. Then we have that:
\begin{equation}
    \text{Regret}(T, \pi_\tau, M^*) 
    \leq 30\phi \tau WG \sqrt{T J (J\log(2) + \log(2N\zeta T^2 / \delta))} \label{eq:finalbound2}
\end{equation}
with probability at least $1 - 3\delta$.

\begin{proof}
Assume $WG \geq 1$, then by \textbf{Thm.}~\ref{appx_thm:main_regret_bound} we have the following.
\begin{align}
    \text{Regret}(T, \pi_\tau, M^*)
    &\leq 
    2\sqrt{K}
        + WG\sqrt{2 T \log(2 / \delta)}
        + \phi \bigg[ 
            5 \tau J + 8\sqrt{J T \log(4 \phi J T^2 / \delta)}
        \bigg] \\
        &\quad \quad + W G \phi 
        \bigg[
            5 \tau J + 4\sqrt{4 J T (J \log(2) - \log(\delta / 2 N \zeta T^2))}
        \bigg]\\
    &\leq 
    \bigg(
        \phi 5 \tau J (1 + WG) 
        + \sqrt{T} \bigg[
            2 + WG \sqrt{2\log(2 / \delta)} \\
            &\quad \quad + \phi 8 \sqrt{J\log(4\phi JT^2 / \delta)}
            + WG\phi 8 \sqrt{J^2 \log(2) + J \log(2N\zeta T^2 / \delta)}
        \bigg]
\end{align}
To combine the two rightmost square root terms, we compare the terms inside the logarithms:
\begin{align}
    2\zeta N &\geq 4\phi J \\
    2 \zeta |\mathcal{A}| \sum_{j=1}^\phi |\text{Val(Pa}[Z_j^h])| = 2\zeta |\mathcal{A}| \phi J &\geq 4 \phi J \\
    |\mathcal{A}| &\geq \frac{2}{\zeta}
\end{align}
Which is true for any non-trivial MDP with more than a single action. Therefore:
\begin{align}
    &\leq
        10\phi J WG \tau + \sqrt{T}
        \bigg[
            2 + WG\sqrt{2\log(2/\delta)}
            + 16 \phi WG \sqrt{J^2 \log(2) + J \log(2 N \zeta T^2 / \delta)}
        \bigg] \\
    &\leq 
        10\phi J WG \tau + 18\phi WG \sqrt{T(J^2 \log(2) + J\log(2N\zeta T^2/\delta))} \\
    &\leq 
        10\phi WG \tau \sqrt{TJ^2} + 18\phi \tau WG \sqrt{TJ (J \log(2) + \log(2N\zeta T^2/\delta))} \\
    &\leq
        30\phi \tau WG \sqrt{T J (J\log(2) + \log(2N\zeta T^2 / \delta))}
\end{align}
\end{proof}

\end{document}